\documentclass[11pt]{article}
\usepackage{epsfig}
\usepackage{amsmath}
\usepackage{amssymb}
\usepackage{amscd}
\usepackage{picinpar}
\usepackage{times}
\usepackage{graphicx}
\usepackage{wrapfig}
\usepackage{xspace}
\usepackage{diagrams}
\usepackage{amsmath,amsfonts,amssymb,amsthm}

\newtheorem{theorem}{Theorem}[section]
\newtheorem{lemma}[theorem]{Lemma}

\newtheorem{definition}[theorem]{Definition}

\begin{document}

\title{Geometric Understanding of Deep Learning}

\author{
Na Lei
\thanks{Dalian University of Technology, Dalian, China. Email: nalei@dlut.edu.cn}
\and
Zhongxuan Luo
\thanks{Dalian University of Technology, Dalian, China. Email: zxluo@dlut.edu.cn}
\and
Shing-Tung Yau
\thanks{Harvard University, Boston, US. Email: yau@math.harvard.edu}
\and
David Xianfeng Gu
\thanks{Harvard University, Boston, US. Email: gu@cmsa.fas.harvard.edu.}
}

\date{}
\maketitle


\begin{abstract}
Deep learning is the mainstream technique for many machine learning tasks, including image recognition, machine translation, speech recognition, and so on. It has outperformed conventional methods in various fields and achieved great successes. Unfortunately, the understanding on how it works remains unclear. It has the central importance to lay down the theoretic foundation for deep learning.

In this work, we give a geometric view to understand deep learning: we show that the fundamental principle attributing to the success is the manifold structure in data, namely natural high dimensional data concentrates close to a low-dimensional manifold, deep learning learns the manifold and the probability distribution on it.

We further introduce the concepts of rectified linear complexity for deep neural network measuring its learning capability, rectified linear complexity of an embedding manifold describing the difficulty to be learned. Then we show for any deep neural network with fixed architecture, there exists a manifold that cannot be learned by the network. Finally, we propose to apply optimal mass transportation theory to control the probability distribution in the latent space.

%
%
\end{abstract}

\newpage
\section{Introduction}

Deep learning is the mainstream technique for many machine learning tasks, including image recognition, machine translation, speech recognition, and so on \cite{Goodfellow2016}. It has outperformed conventional methods in various fields and achieved great successes. Unfortunately, the understanding on how it works remains unclear. It has the central importance to lay down the theoretic foundation for deep learning.

We believe that the main fundamental principle to explain the success of deep learning is the \textbf{manifold structure} in the data, there exists a well accepted manifold assumption: \emph{natural high dimensional data concentrates close to a non-linear low-dimensional manifold.}

\paragraph*{Manifold Representation} The main focus of various deep learn methods is to learn the manifold structure from the real data and obtain a  parametric representation of the manifold.
In general, there is a probability distribution $\mu$ in the ambient space $\mathcal{X}$, the support of $\mu$ is a low dimensional manifold $\Sigma\subset \mathcal{X}$. For example, an autoencoder learns the encoding map $\varphi_\theta:\mathcal{X}\to\mathcal{F}$ and the decoding map $\psi_\theta:\mathcal{F}\to \mathcal{X}$, where $\mathcal{F}$ is the latent space. The parametric representation of the input manifold $\Sigma$ is given by the decoding map $\psi_\theta$. The reconstructed manifold $\tilde{\Sigma}=\psi_\theta\circ\varphi_\theta(\Sigma)$ approximates input manifold. Furthermore, the DNN also learns and controls the distribution induced by the encoder $(\varphi_\theta)_*\mu$ defined on the latent space. Once the parametric manifold structure is obtained, it can be applied for various application, such as randomly generating a sample on $\tilde{\Sigma}$ as a generative model. Image denoising can be reinterpreted geometrically as projecting a noisy sample onto $\tilde{\Sigma}$ representing the clean image manifold, the closest point on $\tilde{\Sigma}$ gives the denoised image.

\paragraph*{Learning Capability} An autoencoder implemented by a ReLU DNN offers a piecewise functional space, the manifold structure can be learned by optimizing  special loss functions. We introduce the concept of \emph{Rectified Linear Complexity} of a DNN, which represents the upper bound of the number of pieces of all the functions representable by the DNN, and gives a measurement for the learning capability of the DNN. On the other hand, the piecewise linear encoding map $\varphi_\theta$ defined on the ambient space is required to be homemorphic from $\Sigma$ to a domain on $\mathcal{F}$. This requirement induces strong topological constraints of the input manifold $\Sigma$. We introduce another concept \emph{Rectified linear Complexity} of an embedded manifold $(\Sigma,\mathcal{X})$, which describes the minimal number of pieces for a PL encoding map, and measures the difficulty to be encoded by a DNN. By comparing the complexities of the DNN and the manifold, we can verify if the DNN can learn the manifold in principle. Furthermore, we show for any DNN with fixed architecture, there exists an embedding manifold that can not be encoded by the DNN.


\paragraph*{Latent Probability Distribution Control} The distribution $(\varphi_\theta)_*\mu$ induced by the encoding map can be controlled by designing special loss functions to modify the encoding map $\varphi_\theta$. We also propose to use optimal mass transportation theory to find the optimal transportation map defined on the latent space, which transforms simple distributions, such as Gaussian or uniform, to $(\varphi_\theta)_*\mu$. Comparing to the conventional WGAN model, this method replaces the blackbox by explicit mathematical construction, and avoids the competition between the generator and the discriminator.
\subsection{Contributions}
This work proposes a geometric framework to understand autoencoder and general deep neural networks and explains the main theoretic reason for the great success of deep learning - the manifold structure hidden in data. The work introduces the concept of rectified linear complexity of a ReLU DNN to measure the learning capability, and rectified linear complexity of an embedded manifold to describe the encoding difficulty. By applying the concept of complexities, it is shown that for any DNN with fixed architecture, there is a manifold too complicated to be encoded by the DNN. Finally, the work proposes to apply optimal mass transportation map to control the distribution on the latent space.

\subsection{Organization}
The current work is organized in the following way: section~\ref{sec:previous_works} briefly reviews the literature of autoencoders; section~\ref{sec:manifold} explains the manifold representation; section~\ref{sec:learning_capability} quantifies the learning capability of a DNN and the learning difficulty for a manifold; section~\ref{sec:measure_control} proposes to control the probability measure induced by the encoder using optimal mass transportation theory. Experimental results are demonstrated in the appendix~\ref{sec:appendix}.

\section{Previous Works}
\label{sec:previous_works}

The literature of autoencoders is vast, in the following we only briefly review the most related ones as representatives.

\paragraph*{Traditional Autoencoders (AE)}

The \emph{traditional autoencoder (AE)} framework first appeared in \cite{Baldi1989NNP}, which was initially proposed to achieve dimensionality reduction. \cite{Baldi1989NNP} use linear autoencoder to compare with PCA. With the same purpose, \cite{HinSal2006DR} proposed a \emph{deep autoencoder} architecture, where the encoder and the decoder are multi-layer deep networks.
Due to non-convexity of deep networks, they are easy to converge to poor local optima with random initialized weights. To solve this problem, \cite{HinSal2006DR} used restricted Boltzmann machines (RBMs) to pre-train the model layer by layer before fine-tuning. Later \cite{Bengio2007SLA} used traditional AEs to pre-train each layer and got similar results.

\paragraph*{Sparse Encoders}

The traditional AE uses bottleneck structure, the width of the middle later is less than that of the input layer. The \emph{sparse autoencoder (SAE)} was introduced in \cite{Foldiak1998SCP}, which uses over-complete latent space, that is the middle layer is wider than the input layer. Sparse autoencoders \cite{OLSHAUSEN19973311,Ranzato2006ELS,Ranzato2007SFL} were proposed.

Extra regularizations for sparsity was added in the object function, such as the KL divergence between the bottle neck layer output distribution and the desired distribution \cite{Ng2011SAE}. SAEs are used in a lot of classification tasks \cite{Xu2016SSAE, Tao2015SSAE}, and feature tranfer learning \cite{Deng2013SAE}.

\paragraph*{Denoising Autoencoder (DAE)}
\cite{Vincent2010SDA,Vincent2008ECR} proposed \emph{denoising autoencoder (DAE)} in order to improve the robustness from the corrupted input. DAEs add regularizations on inputs to reconstruct a ``repaired'' input from a corrupted version. \emph{Stacked denoising autoencoders (SDAEs)} is constructed by stacking multiple layers of DAEs, where each layer is pre-trained by DAEs. The DAE/SDAE is suitable for denosing purposes, such as speech recognition \cite{Feng2014SFD,Feng2014SFD}, and removing musics from speeches \cite{Zhao2015MR}, medical image denoising \cite{Gondara2016MID} and  super-resolutions \cite{Chaitanya2017IRM}.

\paragraph*{Contractive Autoencoders (CAEs)} \cite{Rifai2011CAE} proposed \emph{contractive autoencoders (CAEs)} to achieve robustness by minimizing the first order variation, the Jacobian. The concept of contraction ratio is introduced, which is similar to the Lipschitz constants. In order to learn the low-dimensional structure of the input data, the panelty of construction error
encourages the contraction ratios on the tangential directions of the manifold to be close to $1$, and on the orthogonal directions to the manifold close to $0$.
Their experiments showed that the learned representations performed as good as DAEs on classification problems and showed that their contraction properties are similar. Following this work, \cite{Rifai2011HOC} proposed \emph{the higher-order CAE} which adds an additional penalty on all higher derivatives.

\paragraph*{Generative Model} Autoencoders can be transformed into a generative model by sampling in the latent space and then decode the samples to obtain new data. \cite{Vincent2010SDA} used Bernoulli sampling to AEs and DAEs to first implement this idea. \cite{Bengio2013GDA} used Gibbs sampling to alternatively sample between the input space and the latent space, and transfered DAEs into generative models. They also proved that the generated distribution is consistent with the distribution of the dataset. \cite{Rifai2012GPS} proposed a generative model by sampling from CADs. They used the information of the Jacobian to sample around the latent space.

The \emph{Variational autoencoder (VAE)} \cite{Kingma2013AutoEncodingVB} use probability perspective to interprete autoencoders. Suppose the real data distribution is $\mu$ in $\mathcal{X}$, the encoding map $\varphi_\theta: \mathcal{X}\to \mathcal{F}$ pushes $\mu$ forward to a distribution in the latent space $(\varphi_\theta)_*\mu$. VAE optimizes $\varphi_\theta$, such that $(\varphi_\theta)_*\mu$ is normal distributed $(\varphi_\theta)_*\mu\sim \mathcal{N}(0,1)$ in the latent space.

Followed by the big success of GANs, \cite{Goodfellow2016AAE} proposed \emph{adversarial autoencoders (AAEs)}, which use GANs to minimize the discrepancy between the push forward distribution $(\varphi_\theta)_*\mu$ and the desired distribution in the latent space.

\section{Manifold Structure}
\label{sec:manifold}
Deep learning is the mainstream technique for many machine learning tasks, including image recognition, machine translation, speech recognition, and so on \cite{Goodfellow2016}. It has outperformed conventional methods in various fields and achieved great successes. Unfortunately, the understanding on how it works remains unclear. It has the central importance to lay down the theoretic foundation for deep learning.

We believe that the main fundamental principle to explain the success of deep learning is the manifold structure in the data, namely \emph{natural high dimensional data concentrates close to a non-linear low-dimensional manifold.}

The goal of deep learning is to learn the manifold structure in data and the probability distribution associated with the manifold.

\subsection{Concepts and Notations}

The concepts related to manifold are from differential geometry, and have been translated to the machine learning language.

\begin{figure}[h]
\begin{center}
\begin{tabular}{c}
\includegraphics[width=0.6\textwidth]{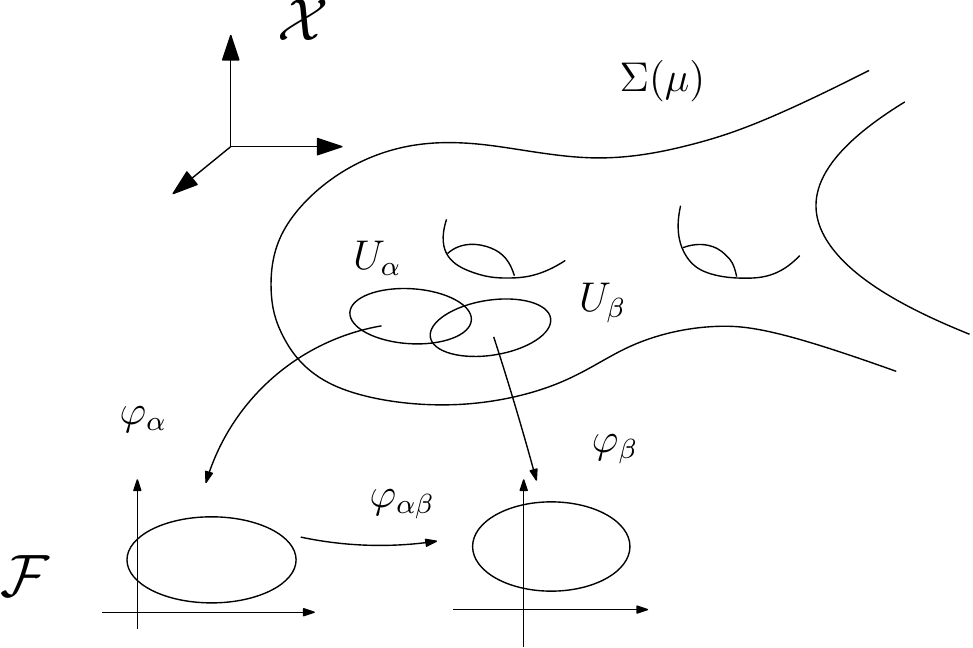}\\
\end{tabular}
\end{center}
\caption{A manifold structure in the data.\label{fig:manifold}}
\end{figure}

\begin{definition}[Manifold]
An $n$-dimensional manifold $\Sigma$ is a topological space, covered by a set of open sets $\Sigma\subset \bigcup_\alpha U_\alpha$.
For each open set $U_\alpha$, there is a homeomorphism $\varphi_\alpha:U_\alpha\to \mathbb{R}^n$, the pair $(U_\alpha,\varphi_\alpha)$ form a chart.
The union of charts form an atlas $\mathcal{A}=\{(U_\alpha,\varphi_\alpha)\}$. If $U_\alpha \cap U_\beta \neq\emptyset$, then the chart transition map is given by
$\varphi_{\alpha\beta}: \varphi_\alpha(U_\alpha\cap U_\beta) \to  \varphi_\beta(U_\alpha\cap U_\beta)$,
\[
   \varphi_{\alpha\beta}:= \varphi_\beta\circ\varphi_\alpha^{-1}.
\]
\end{definition}

\begin{figure}[h!]
\begin{center}
\begin{tabular}{cccc}
\includegraphics[height=0.33\textwidth]{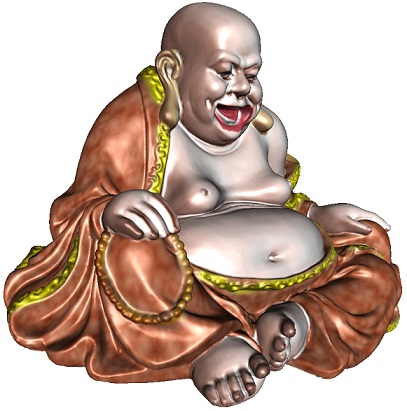}&
\includegraphics[width=0.33\textwidth]{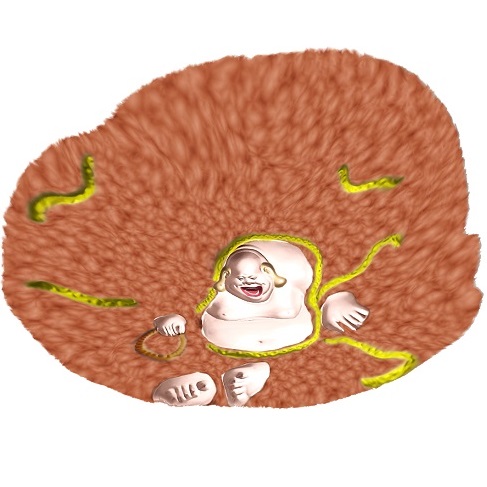}&
\includegraphics[height=0.32\textwidth]{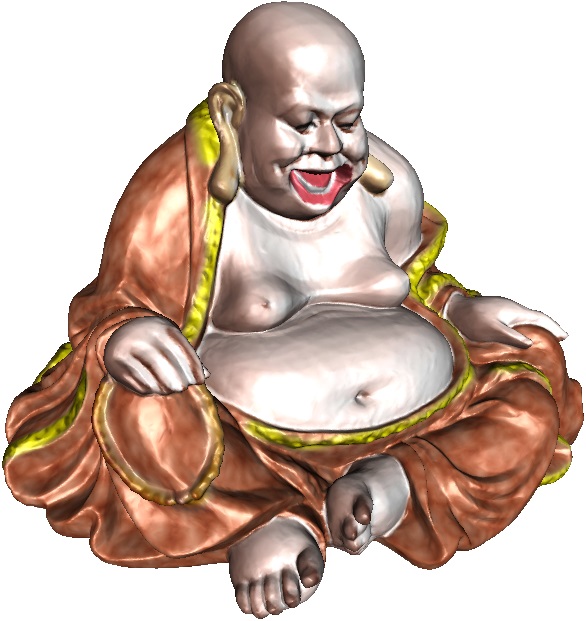}\\
a. Input manifold &b. latent representation & c. reconstructed manifold \\
$M\subset\mathcal{X}$ & $D=\varphi_\theta(M)$ & $\tilde{M}=\psi_\theta(D)$\\
\end{tabular}
\end{center}
\caption{Auto-encoder pipeline.}
\label{fig:autoencoder_pipeline}
\end{figure}

As shown in Fig.~\ref{fig:manifold}, suppose $\mathcal{X}$ is the \emph{ambient space}, $\mu$ is a probability distribution defined on $\mathcal{X}$, represented as a density function $\mu:\mathcal{X}\to\mathbb{R}_{\ge 0}$. The \emph{support} of $\mu$,
\[
    \Sigma(\mu) :=\{\mathbf{x}\in \mathcal{X}| \mu(x)>0\}
\]
is a low-dimensional manifold. $(U_\alpha,\varphi_\beta)$ is a local chart, $\varphi_\alpha:U_\alpha\to \mathcal{F}$ is called an \emph{encoding map}, the parameter domain $\mathcal{F}$ is called the \emph{latent space} or \emph{feature space}. A point $\mathbf{x}\in \Sigma$ is called a \emph{sample}, its parameter $\varphi_\alpha(\mathbf{x})$ is called the \emph{code} or \emph{feature} of $\mathbf{x}$. The inverse map $\psi_\alpha:=\varphi_\alpha^{-1}:\mathcal{F}\to\Sigma$ is called the \emph{decoding map}. Locally, $\psi_\alpha:\mathcal{F}\to\Sigma$ gives a local \emph{parametric representation} of the manifold.

Furthermore, the encoding map $\varphi_\alpha:U_\alpha\to \mathcal{F}$ induces a \emph{push-forward} probability measure $(\varphi_\alpha)_*\mu$ defined on the latent space $\mathcal{F}$: for any measurable set $B\subset \mathcal{F}$,
\[
    (\varphi_\alpha)_*\mu (B) := \mu(\varphi_\alpha(B)).
\]

The goal for deep learning is to learn the encoding map $\varphi_\alpha$, decoding map $\psi_\alpha$, the parametric representation of the manifold $\psi_\alpha:\mathcal{F}\to \Sigma$, furthermore the push-forward probability $(\varphi_\alpha)_*\mu$ and so on. In the following, we explain how an autoencoder learns the manifold and the distribution.

\subsection{Manifold Learned by an Autoencoder}

Autoencoders are commonly used for unsupervised learning \cite{Bengio2013RLR}, they have been applied for compression, denoising, pre-training and so on.
In abstract level, autoencoder learns the low-dimensional structure of data and represent it as a parametric polyhedral manifold, namely a piecewise linear (PL) map from latent space (parameter domain) to the ambient space, the image of the PL mapping is a manifold. Then autoencoder utilizes the polyhedral manifold as the approximation of the manifold in data for various applications. In implementation level, an autoencoder partition the manifold into pieces (by decomposing the ambient space into cells) and approximate each piece by a hyperplane as shown in Fig.~\ref{fig:autoencoder_pipeline}.

Architecturally, an autoencoder is a feedforward, non-recurrent neural network with the output layer having the same number of nodes as the input layer, and with the purpose of reconstructing its own inputs. In general, a bottleneck layer is added for the purpose of dimensionality reduction. The input space $\mathcal{X}$ is the ambient space, the output space is also the ambient space. The output space of the bottle neck layer $\mathcal{F}$ is the latent space.
\begin{diagram}
 \{(\mathcal{X},\mathbf{x}),\mu,\Sigma\} &\rTo^{\varphi} &\{(\mathcal{F},\mathbf{z}), D\} \\
&\rdTo_{\psi \circ \varphi} &\dTo^{\psi} \\
& &\{(\mathcal{X},\mathbf{\tilde{x}}),\tilde{\Sigma}\}
\end{diagram}
An autoencoder always consists of two parts, the encoder and the decoder. The encoder takes a sample $\mathbf{x}\in \mathcal{X}$ and maps it to $\mathbf{z}\in \mathcal{F}$, $\mathbf{z}=\varphi(\mathbf{x})$, the image $\mathbf{z}$ is usually referred to as \emph{latent representation} of $\mathbf{x}$. The encoder $\varphi:\mathcal{X}\rightarrow \mathcal{F}$ maps $\Sigma$ to its latent representation $D=\varphi(\Sigma)$ homemorphically. After that, the decoder $\psi:\mathcal{F}\rightarrow \mathcal{X}$ maps $\mathbf{z}$ to the reconstruction $\mathbf{\tilde{x}}$  of the same shape as $\mathbf{x}$, $\mathbf{\tilde{x}} = \psi(\mathbf{z}) = \psi\circ\varphi(\mathbf{x})$. Autoencoders are also trained to minimise reconstruction errors:
\[
    \varphi,\psi = \text{argmin}_{\varphi,\psi} \int_\mathcal{X} \mathcal{L}(\mathbf{x},\psi\circ\varphi(\mathbf{x})) d\mu(\mathbf{x}),
\]
where $\mathcal{L}(\cdot,\cdot)$ is the loss function, such as squared errors. The reconstructed manifold $\tilde{\Sigma}=\psi\circ\varphi(\Sigma)$ is used as an approximation of $\Sigma$.

In practice, both encoder and decoder are implemented as ReLU DNNs, parameterized by $\theta$. Let $X=\{\mathbf{x}^{(1)},\mathbf{x}^{(2)},\dots, \mathbf{x}^{k}\}$ be the training data set, $X\subset \Sigma$, the autoencoder optimizes the following loss function:
\[
    \min_\theta\mathcal{L}(\theta) = \min_\theta\frac{1}{k}\sum_{i=1}^k \| \mathbf{x}^{(i)}-\psi_\theta\circ\varphi_\theta(\mathbf{x}^{(i)})\|^2.
\]
Both the encoder $\varphi_\theta$ and the decoder $\psi_\theta$ are piecewise linear mappings. The encoder $\varphi_\theta$ induces a cell decomposition $\mathcal{D}(\varphi_\theta)$ of the ambient space
\[
    \mathcal{D}(\varphi_\theta): \mathcal{X} = \bigcup_\alpha U_\theta^\alpha,
\]
where $U_\theta^\alpha$ is a convex polyhedron, the restriction of $\varphi_\theta$ on it is an affine map. Similarly, the piecewise linear map $\psi_\theta\circ\varphi_\theta$ induces a polyhedral cell decomposition $\mathcal{D}(\psi_\theta,\varphi_\theta)$, which is a refinement (subdivision) of $\mathcal{D}(\varphi_\theta)$. The reconstructed polyhedral manifold has a parametric representation $\psi_\theta:\mathcal{F}\to\mathcal{X}$, which approximates the manifold $M$ in the data.

Fig.~\ref{fig:autoencoder_pipeline} shows an example to demonstrate the learning results of an autoencoder. The ambient space $\mathcal{X}$ is $\mathbb{R}^3$, the manifold $\Sigma$ is the buddha surface as shown in frame (a). The latent space is $\mathbb{R}^2$, the encoding map $\varphi_\theta:\mathcal{X}\to D$ parameterizes the input manifold to a domain on $D\subset\mathcal{F}$ as shown in frame (b). The decoding map $\psi_\theta:D\to \mathcal{X}$ reconstructs the surface into a piecewise linear surface $\tilde{\Sigma}=\psi_\theta\circ\varphi_\theta(\Sigma)$, as shown in frame (c). In ideal situation, the composition of the encoder and decoder $\psi_\theta\circ\varphi_\theta\sim id$ should equal to the identity map, the reconstruction $\tilde{\Sigma}$ should coincide with the input $\Sigma$. In reality, the reconstruction $\tilde{\Sigma}$ is only a piecewise linear approximation of $\Sigma$.

\begin{figure}[h!]
\begin{center}
\begin{tabular}{cccc}
\includegraphics[height=0.33\textwidth]{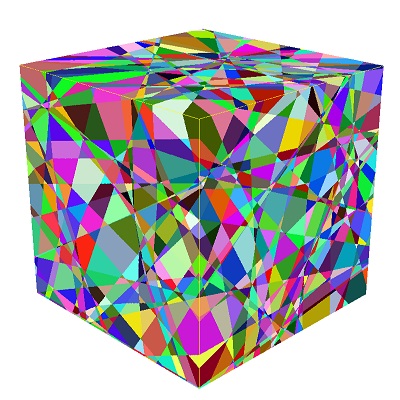}&
\includegraphics[height=0.3\textwidth]{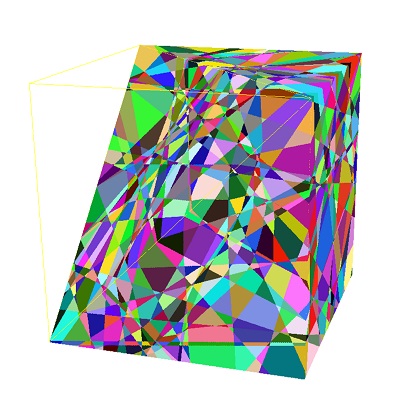}&
\includegraphics[height=0.32\textwidth]{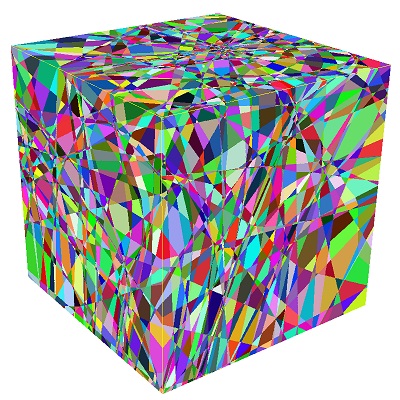}\\
d. cell decomposition&e. cut view of  &f. cell decomposition &\\
 $\mathcal{D}(\varphi_\theta)$&$\mathcal{D}(\varphi_\theta)$  &$\mathcal{D}(\psi_\theta\circ\varphi_\theta)$&\\
\end{tabular}
\end{center}
\caption{Cell decomposition induced by the encoding, decoding maps.}
\label{fig:cell_decomposition}
\end{figure}

Fig.~\ref{fig:cell_decomposition} shows the cell decompositions induced by the encoding map $\mathcal{D}(\varphi_\theta)$ and that by the reconstruction map $\mathcal{D}(\psi_\theta\circ\varphi_\theta)$ for another autoencoder. It is obvious that $\mathcal{D}(\psi_\theta\circ\varphi_\theta)$ subdivides $\mathcal{D}(\varphi_\theta)$.

\subsection{Direct Applications}

Once the neural network has learned a manifold $\Sigma$, it can be utilized for many applications.

\paragraph*{Generative Model}


Suppose $\mathcal{X}$ is the space of all $n\times n$ color images, where each point represents an image. We can define a probability measure $\mu$, which represents the probability for an image to represent a human face. The shape of a human face is determined by a finite number of genes. The facial photo is determined by  the geometry of the face, the lightings, the camera parameters and so on. Therefore, it is sensible to assume all the human facial photos are concentrated around a finite dimensional manifold, we call it as human facial photo manifold $\Sigma$.

By using many real human facial photos, we can train an autoendoer to learn the human facial photo manifold. The learning process produces a decoding map $\psi_\theta:\mathcal{F}\to \tilde{\Sigma}$, namely a parametric representation of the reconstructed manifold. We randomly generate a parameter $z\in\mathcal{F}$ (white noise), $\varphi_\theta(z)\in \tilde{\Sigma}$ gives a human facial image. This can be applied as a generative model for generating human facial photos.

\paragraph*{Denoising}
Tradition image denoising performs Fourier transformation of the input noisy image, then filtering out the high frequency components, inverse Fourier transformation to get the denoised image. This method is general and independent of the content of the image.
\begin{figure}[h!]
\begin{center}
\begin{tabular}{c}
\includegraphics[height=0.32\textwidth]{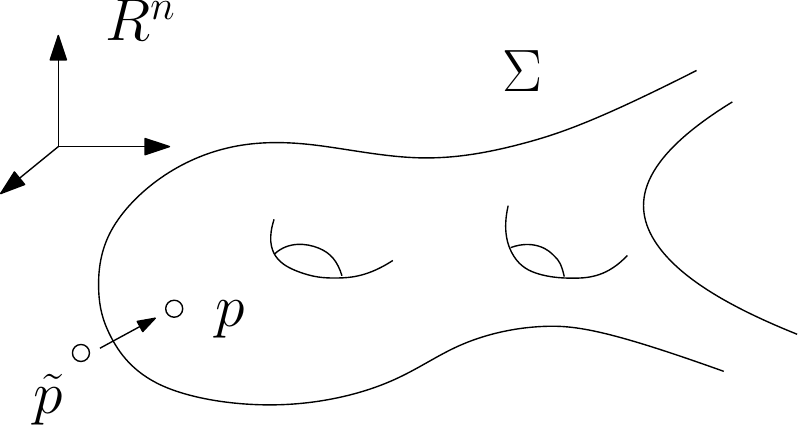}\\
\end{tabular}
\end{center}
\caption{Geometric interpretation of image denoising. \label{fig:denoise}}
\end{figure}

In deep learning, image denoising can be re-interpreted as geometric projection as shown in Fig.~\ref{fig:denoise}. Suppose we perform human facial image denoising. The clean human facial photo manifold is $\Sigma$, the noisy facial image $\tilde{p}$ is not in $\Sigma$ but close to $\Sigma$. We project $\tilde{p}$ to $\Sigma$, the closest point to $\tilde{p}$ on $\Sigma$ is $p$, then $p$ is the denoised image.

\begin{figure}[h!]
\begin{center}
\begin{tabular}{cc}
\includegraphics[height=0.3\textwidth]{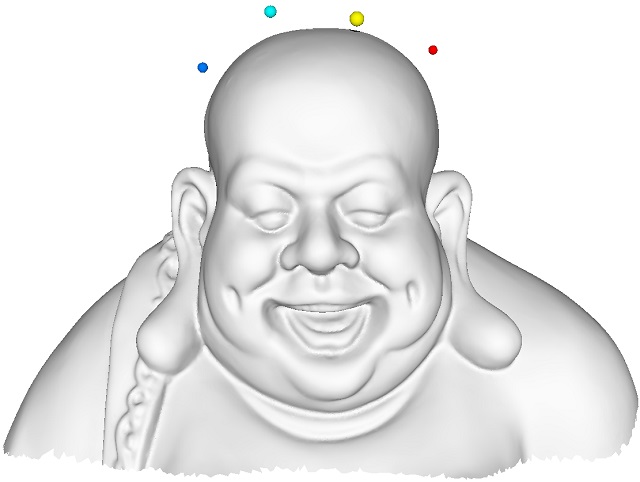}&
\includegraphics[height=0.3\textwidth]{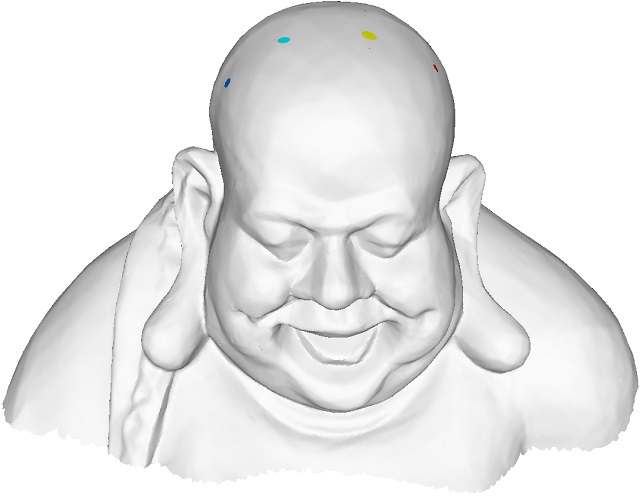}\\
a. input manifold &b. reconstructed manifold\\
\end{tabular}
\end{center}
\caption{Geometric projection. \label{fig:projection}}
\end{figure}

In practice, suppose an noisy facial image is given $\mathbf{x}$, we train an autoencoder to obtain a manifold of clean facial images represented as $\psi_\theta:\mathcal{F}\to\mathcal{X}$ and an encoding map $\varphi_\theta:\mathcal{X}\to\mathcal{F}$, then we encode the noisy image $\mathbf{z}=\varphi(\mathbf{x})$, then maps $\mathbf{z}$ to the reconstructed manifold $\mathbf{\tilde{x}}=\psi_\theta(\mathbf{z})$. The result $\mathbf{\tilde{x}}$ is the denoised image. Fig.~\ref{fig:projection} shows the projection of several outliers onto the buddha surface using an autoencoder.
\setlength{\tabcolsep}{8pt}
\begin{figure}[h!]
\begin{center}
\begin{tabular}{cc}
\includegraphics[height=0.4\textwidth]{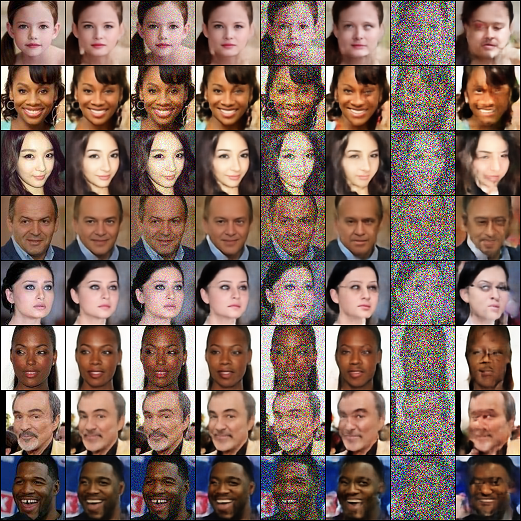}&
\includegraphics[height=0.4\textwidth]{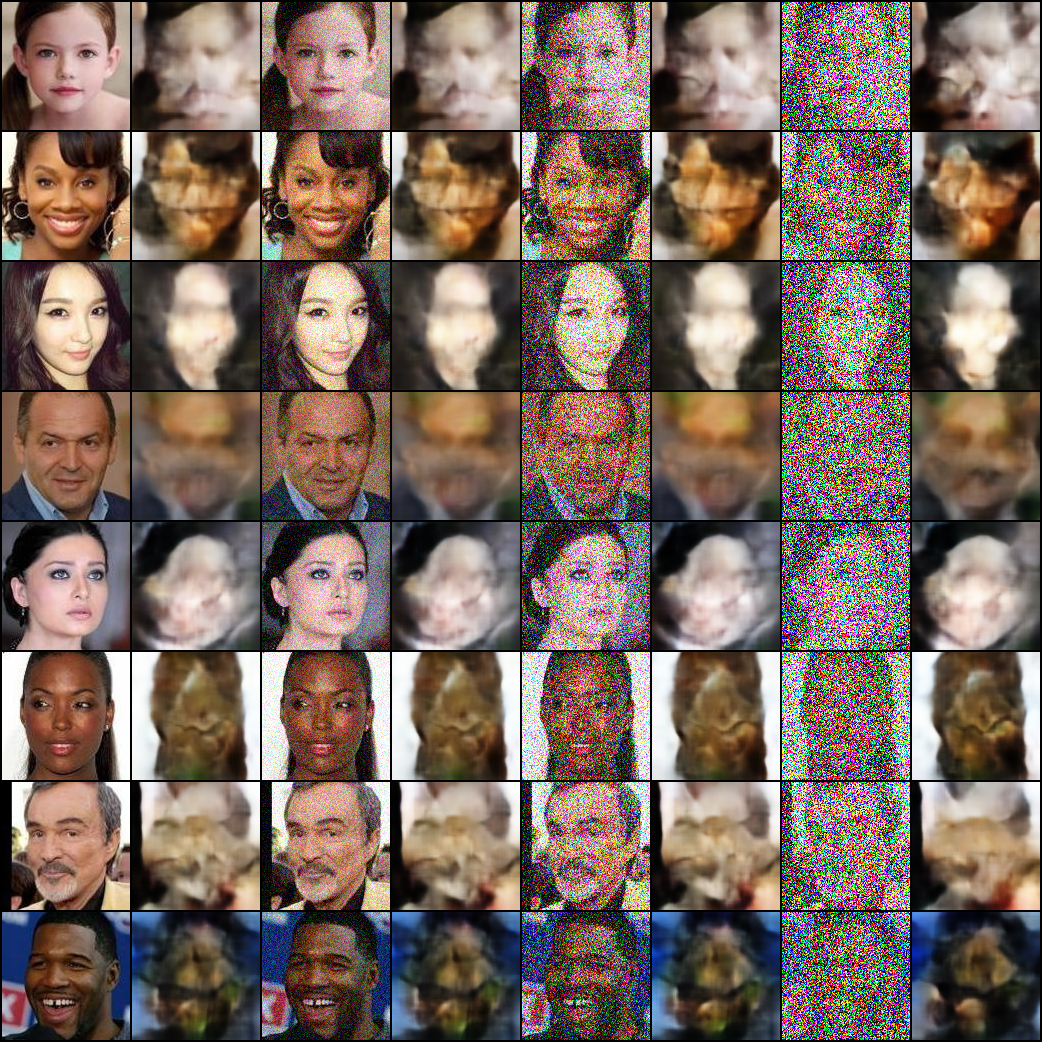}\\
(a) project to the human facial & (b) project to the cat facial \\
image manifold& image manifold\\
\end{tabular}
\end{center}
\caption{Image denoising. \label{fig:face_denoise}}
\end{figure}

We apply this method for human facial image denoising as shown in Fig.~\ref{fig:face_denoise}, in frame (a) we project the noisy image to the human facial image manifold and obtain good denoising result; in frame (b) we use the cat facial image manifold, the results are meaningless. This shows deep learning method heavily depends on the underlying manifold, which is specific to the problem. Hence the deep learning based method is not as universal as the conventional ones.

\section{Learning Capability}
\label{sec:learning_capability}

\setlength{\tabcolsep}{0pt}
\begin{figure}[h!]
\begin{center}
\begin{tabular}{ccc}
\includegraphics[height=0.26\textwidth]{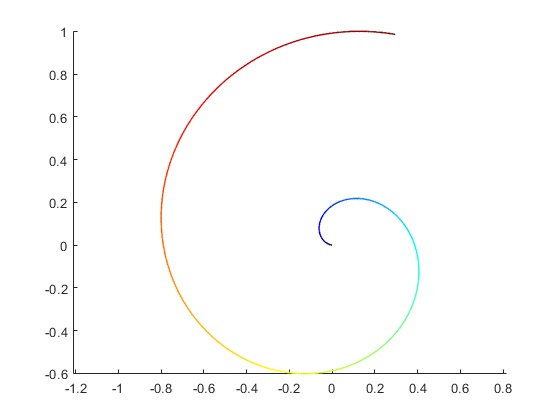}&
\includegraphics[height=0.26\textwidth]{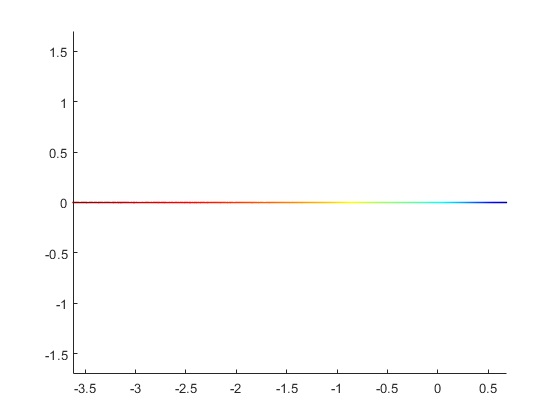}&
\includegraphics[height=0.26\textwidth]{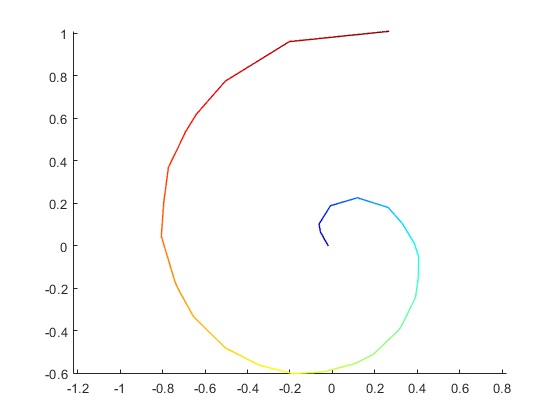}\\
a. Input manifold & b. latent representation & c. reconstructed manifold \\
$M\subset\mathcal{X}$ & $D=\varphi_\theta(M)$ & $\tilde{M}=\psi_\theta(D)$\\
\includegraphics[height=0.26\textwidth]{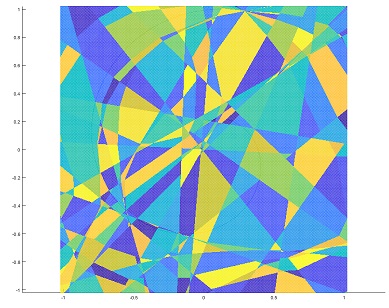}&
\includegraphics[height=0.26\textwidth]{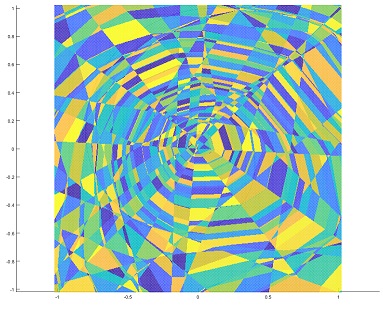}&
\includegraphics[height=0.26\textwidth]{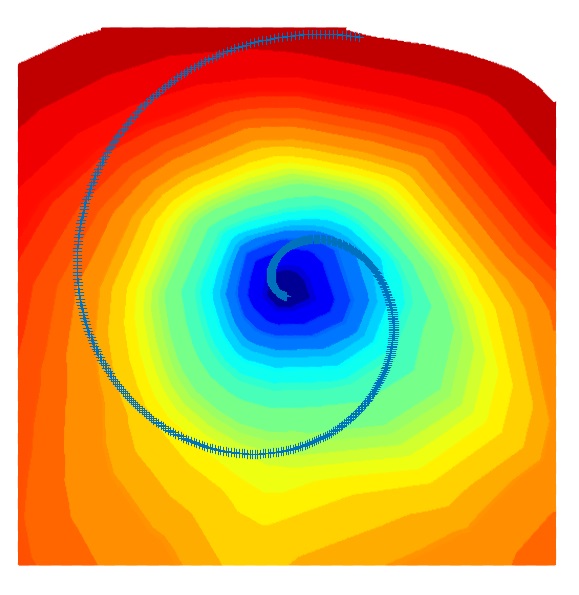}\\
d. cell decomposition & e. cell decomposition & f. level set\\
 $\mathcal{D}(\varphi_\theta)$& $\mathcal{D}(\psi_\theta\circ\varphi_\theta)$&\\
\end{tabular}
\end{center}
\caption{Encode/decode a spiral curve.\label{fig:spiral}}
\end{figure}
\subsection{Main Ideas}
Fig.~\ref{fig:spiral} shows another example, an Archimedean spiral curve embedded in $\mathbb{R}^2$, the curve equation is given by $\rho(\theta) = (a+b\theta)e^{iw\theta}$, $a,b,w>0$ are constants, $\theta\in (0,T]$. For relatively small range $T$, the encoder successfully maps it onto a straight line segment, and the decoder reconstructs a piecewise linear curve with good approximation quality. When we extend the spiral curve by enlarging $T$, then at some threshold, the autoencoder with the same architecture fails to encode it.

The central problems we want to answer are as follows:
\begin{enumerate}
\item How to decide the bound of the encoding or representation capability for an autoencoder with a fixed ReLU DNN architecture?
\item How to describe and compute the complexity of a manifold embedded in the ambient space to be encoded ?
\item How to verify whether a embedded manifold can be encoded by a ReLU DNN autoencoder?
\end{enumerate}

For the first problem, our solutions are based on the geometric intuition of the piecewise linear nature of encoder/decoder maps. By examining fig.~\ref{fig:cell_decomposition} and fig.~\ref{fig:spiral}, we can see the mapping $\varphi_\theta$ and $\psi_\theta$ induces polyhedral cell decompositions of the ambient space $\mathcal{X}$,
$\mathcal{D}(\varphi_\theta)$ and $\mathcal{D}(\psi_\theta\circ\varphi_\theta)$ respectively. The number of cells offers a measurement to describing the representation capabilities of these maps, the upper bound of the number of cells
$\max_\theta |\mathcal{D}(\varphi_\theta)|$ describes the limit of the encoding capability of $\varphi_\theta$. We call this upper bound as the \emph{rectified linear complexity} of the autoencoder. The rectified linear complexity can be deduced from the architecture of the encoder network, as claimed in our theorem \ref{thm:DNN_complexity}.

For the second problem, we introduce the similar concept to the embedded manifold. The encoder map $\varphi_\theta$ has a very strong geometric requirement: suppose $U_k$ is a cell in $\mathcal{D}(\varphi_\theta)$, then $\varphi_\theta:U_k\to\mathcal{F}$ is an affine map to the latent space, its restriction on $U_k\cap \Sigma$ is a homeomorphism $\varphi_\theta: U_k\cap \Sigma\to \varphi_\theta(U_k\cap \Sigma)$. In order to satisfy the two stringent requirements for the encoding map: the piecewise ambient linearity and the local homeomorphism, the number of cells of the decomposition of $\Sigma$ (and of $\mathcal{X}$) must be greater than a lower bound. Similarly, we call this lower bound the \emph{rectified linear complexity} of the pair of the manifold and the ambient space $(\mathcal{X},\Sigma)$. The rectified linear complexity can be derived from the geometry of $\Sigma$ and its embedding in $\mathcal{X}$. Our theorem \ref{thm:encodable_condition} gives a criteria to verify if a manifold can be rectified by a linear map.

For the third problem, we can compare the rectified linear complexity of the manifold and the autoencoder. If the RL complexity of the autoencoder is less than that of the manifold, then the autoencoder can not encode the manifold. Specifically, we show that for any autoencoder with a fixed architecture, there exists an embedded manifold, which can not be encoded by it.

\subsection{ReLU Deep Neuron Networks}
We extend the ReLU activation function to vectors $\mathbf{x}\in \mathbb{R}^n$ through entry-wise operation:
\[
    \sigma(x)=(\max\{0,x_1\},\max\{0,x_2\},\dots,\max\{0,x_n\}).
\]
For any $(m,n)\in\mathbb{N}$, let $\mathcal{A}_m^n$ and $\mathcal{L}_m^n$ denote the class of affine and linear transformations from $\mathbb{R}^m\to\mathbb{R}^n$, respectively.
\begin{definition}[ReLU DNN] For any number of hidden layers $k\in\mathbb{N}$, input and output dimensions $w_0,w_{k+1}\in \mathbb{N}$, a $\mathbb{R}^{w_0}\to \mathbb{R}^{w_{k+1}}$ ReLU DNN is given by specifying a sequence of $k$ natural numbers $w_1,w_2,\dots, w_k$ representing widths of the hidden layers, a set of $k$ affine transformations $T_i:\mathbb{R}^{w_{i-1}}\to\mathbb{R}^{w_i}$ for $i=1,\dots, k$ and a linear transformation $T_{k+1}:\mathbb{R}^{w_k}\to \mathbb{R}^{w_{k+1}}$ corresponding to weights of hidden layers. Such a ReLU DNN is called a $(k+1)$-layer ReLU DNN, and is said to have $k$ hidden layers, denoted as $N(w_0,w_1,\dots,w_k, w_{k+1})$.
\end{definition}
The mapping $\varphi_\theta:\mathbb{R}^{w_0}\to \mathbb{R}^{w_{k+1}}$ represented by this ReLU DNN is
\begin{equation}
\varphi_\theta = T_{k+1}\circ \sigma \circ T_k \circ \cdots \circ T_2\circ \sigma \circ T_1,
\end{equation}
where $\circ$ denotes mapping composition, $\theta$ represent all the weight and bias parameters. The depth of the ReLU DNN is $k+1$, the width is $\max\{w_1,\dots,w_k\}$, the size $w_1+w_2+ \dots + w_k$.

\begin{definition}[PL Mapping] A mapping $\varphi:\mathbb{R}^n\to \mathbb{R}^m$ is a piecewise linear mapping if there exists a finite set of polyhedra whose union is $\mathbb{R}^n$, and $\varphi$ is affine linear over each polyhedron. The number of pieces of $\varphi$ is the number of maximal connected subsets of $\mathbb{R}^n$ over which $\varphi$ is affine linear, denoted as $\mathcal{N}(\varphi)$. We call $\mathcal{N}(\varphi)$ as the \emph{rectified linear complexity} of $\varphi$.
\end{definition}

\begin{definition}[Rectified Linear Complexity of a ReLU DNN] Given a ReLU DNN $N(w_0,\dots,w_{k+1})$, its rectified linear complexity is the upper bound of the rectified linear complexities of all PL functions $\varphi_\theta$ represented by $N$,
\[
    \mathcal{N}(N):= \max_\theta \mathcal{N}(\varphi_\theta).
\]
\end{definition}

\begin{lemma}The maximum number of parts one can get when cutting $d$-dimensional space $\mathbb{R}^d$ with $n$ hyperplanes is denoted as $\mathcal{C}(d,n)$, then
\begin{equation}
    \mathcal{C}(d,n) = \left(
    \begin{array}{c}
    n\\
    0
    \end{array}
    \right) +
    \left(
    \begin{array}{c}
    n\\
    1
    \end{array}
    \right) +
    \left(
    \begin{array}{c}
    n\\
    2
    \end{array}
    \right) + \cdots +
    \left(
    \begin{array}{c}
    n\\
    d
    \end{array}
    \right).
    \label{eqn:partition_cells}
\end{equation}
\end{lemma}
\begin{proof}Suppose $n$ hyperplanes cut $\mathbb{R}^d$ into $\mathcal{C}(d,n)$ cells, each cell is a convex polyhedron. The $(n+1)$-th hyperplane is $\pi$, then the first $n$ hyperplanes intersection $\pi$ and partition $\pi$ into $\mathcal{C}(d-1,n)$ cells, each cell on $\pi$ partitions a polyhedron in $\mathbb{R}^d$ into $2$ cells, hence we get the formula
\[
    \mathcal{C}(d,n+1) = \mathcal{C}(d,n) + \mathcal{C}(d-1,n).
\]
It is obvious that $\mathcal{C}(2,1)=2$, the formula (\ref{eqn:partition_cells}) can be easily obtained by induction.
\end{proof}

\begin{theorem}[Rectified Linear Complexity of a ReLU DNN] Given a ReLU DNN $N(w_0,\dots,w_{k+1})$, representing PL mappings $\varphi_\theta:\mathbb{R}^{w_0}\to \mathbb{R}^{w_{k+1}}$ with $k$ hidden layers of widths $\{w_i\}_{i=1}^k$, then the linear rectified complexity of $N$ has an upper bound,
\begin{equation}
    \mathcal{N}(N) \le \Pi_{i=1}^{k+1}\mathcal{C}(w_{i-1},w_{i}).
    \label{eqn:DNN_complexity}
\end{equation}
\label{thm:DNN_complexity}
\end{theorem}
\begin{proof}
The $i$-th hidden layer computes the mapping $T_i: \mathbb{R}^{w_{i-1}}\to\mathbb{R}^{w_i}$. Each neuron represents a hyperplane in $\mathbb{R}^{w_{i-1}}$, the $w_i$ hyperplanes partition the whole space into $\mathcal{C}(w_{i-1},w_i)$ polyhedra.

The first layer partitions $\mathbb{R}^{w_0}$ into at most $\mathcal{C}(w_0,w_1)$ cells; the second layer further subdivides the cell decomposition, each cell is at most subdivides into $\mathcal{C}(w_1,w_2)$ polyhedra, hence two layers partition the source space into at most $\mathcal{C}(w_0,w_1)\mathcal{C}(w_1,w_2)$. By induction, one can obtain the upper bound of $\mathcal{N}(N)$ as described by the inequality (\ref{eqn:partition_cells}).
\end{proof}

\subsection{Cell Decomposition}
The PL mappings induces cell decompositions of both the ambient space $\mathcal{X}$ and the latent space $\mathcal{F}$. The number of cells is closely related to the rectified linear complexity.

Fix the encoding map $\varphi_\theta$ , let the set of all neurons in the network is denoted as $\mathcal{S}$, all the subsets is denoted as $2^{\mathcal{S}}$.
\begin{definition}[Activated Path]
Given a point $\mathbf{x}\in \mathcal{X}$, the \emph{activated path} of $\mathbf{x}$ consists all the activated neurons when $\varphi_\theta(\mathbf{x})$ is evaluated, and denoted as $\rho(\mathbf{x})$. Then the activated path defines a set-valued function $\rho: \mathcal{X}\to 2^{\mathcal{S}}$.
\end{definition}

\begin{definition}[Cell Decomposition]Fix an encoding map $\varphi_\theta$ represented by a ReLU RNN, two data points $\mathbf{x}_1,\mathbf{x}_2\in \mathcal{X}$ are \emph{equivalent}, denoted as $\mathbf{x}_1\sim\mathbf{x}_2$, if they share the same activated path, $\rho(\mathbf{x}_1)=\rho(\mathbf{x}_2)$. Then each equivalence relation partitions the ambient space $\mathcal{X}$ into cells,
\[
    \mathcal{D}(\varphi_\theta) : \mathcal{X}=\bigcup_\alpha U_\alpha,
\]
each equivalence class corresponds to a cell: $\mathbf{x}_1,\mathbf{x}_2\in U_\alpha$ if and only if $\mathbf{x}_1\sim\mathbf{x}_2$. $\mathcal{D}(\varphi_\theta)$ is called the cell decomposition induced by the encoding map $\varphi_\theta$.
\end{definition}

Furthermore, $\varphi_\theta$ maps the cell decomposition in the ambient space $\mathcal{D}(\varphi_\theta)$ to a cell decomposition in the latent space. Similarly, the composition of the encoding and decoding maps also produces a cell decomposition,  denoted as $\mathcal{D}(\psi_\theta\circ\varphi_\theta)$, which subdivises  $\mathcal{D}(\varphi_\theta)$. Fig.~\ref{fig:autoencoder_pipeline} bottom row shows these cell decompositions.

\subsection{Learning Difficulty}

\begin{definition}[Linear Rectifiable Manifold] Suppose $\Sigma$ is a $m$-dimensional manifold, embedded in $\mathbb{R}^n$, we say $\Sigma$ is linear rectifiable, if there exists an affine map $\varphi: \mathbb{R}^n\to \mathbb{R}^m$, such that the restriction of $\varphi$ on $\Sigma$, $\varphi|_\Sigma:\Sigma\to\varphi(\Sigma)\subset \mathbb{R}^m$, is homeomorphic. $\varphi$ is called the corresponding rectified linear  map of $M$.
\end{definition}

\begin{definition}[Linear Rectifiable Atlas] Suppose $\Sigma$ is a $m$-dimensional manifold, embedded in $\mathbb{R}^n$, $\mathcal{A}=\{(U_\alpha,\varphi_\alpha\}$ is an atlas of $M$. If each chart $(U_\alpha,\varphi_\alpha)$ is linear rectifiable, $\varphi_\alpha:U_\alpha\to \mathbb{R}^m$ is the rectified linear map of $U_\alpha$, then the atlas is called a linear rectifiable atlas of $\Sigma$.
\end{definition}

Given a compact manifold $\Sigma$ and its atlas $\mathcal{A}$, one can select a finite number of local charts $\mathcal{\tilde{A}}=\{(U_i,\varphi_i)\}_{i=1}^n$, $\mathcal{\tilde{A}}$ still covers $\Sigma$. The number of charts of an atlas $\mathcal{A}$ is denoted as $|\mathcal{A}|$.

\begin{definition}[Rectified Linear Complexity of a Manifold] Suppose $\Sigma$ is a $m$-dimensional manifold embedded in $\mathbb{R}^n$, the rectified linear complexity of $\Sigma$ is denoted as $\mathcal{N}(\mathbb{R}^n,\Sigma)$ and defined as,
\begin{equation}
    \mathcal{N}(\mathbb{R}^n,\Sigma) := \min \left\{|\mathcal{A}|~~|\mathcal{A}~\text{is a linear rectifiable~altas~of~}\Sigma\right\}.
\end{equation}
\end{definition}

\subsection{Learnable Condition}

\begin{definition}[Encoding Map] Suppose $M$ is a $m$-dimensional manifold, embedded in $\mathbb{R}^n$, a continuous mapping $\varphi:\mathbb{R}^n\to\mathbb{R}^m$ is called an encoding map of $(\mathbb{R}^n,\Sigma)$, if restricted on $\Sigma$, $\varphi|_\Sigma:\Sigma\to \varphi(\Sigma)\subset \mathbb{R}^m$ is homeomorphic.
\end{definition}

\begin{theorem}
Suppose a ReLU DNN $N(w_0,\dots,w_{k+1})$ represents a PL mapping $\varphi_\theta:\mathbb{R}^n\to \mathbb{R}^m$,
$\Sigma$ is a $m$-dimensional manifold embedded in $\mathbb{R}^n$. If $\varphi_\theta$ is an encoding mapping of $(\mathbb{R}^n,\Sigma)$, then the rectified linear complexity of $N$ is no less that the rectified linear complexity of $(\mathbb{R}^n,\Sigma)$,
\[
    \mathcal{N}(\mathbb{R}^n,\Sigma) \le \mathcal{N}(\varphi_\theta) \le \mathcal{N}(N).
\]
\label{thm:encodable_condition}
\end{theorem}
\begin{proof}
The ReLU DNN computes the PL mapping $\varphi_\theta$, suppose the corresponding cell decomposition of $\mathbb{R}^n$ is
\[
   \mathcal{D}(\varphi_\theta): \mathbb{R}^n = \bigcup_{i=1}^k U_i,
\]
where each $U_i$ is a convex polyhedron, $k\le \mathcal{N}(\varphi_\theta)$. If $\varphi_\theta$ is an encoding map of $\Sigma$, then
\[
    \mathcal{A}:=\{(D_i,\varphi_\theta|_{D_i})| D_i \cap \Sigma\neq \emptyset\}
\]
form a linear rectifiable atlas of $\Sigma$. Hence from the definition of rectified linear complexity of an ReLU DNN and the manifold, we obtain
\[
    \mathcal{N}(\mathbb{R}^n, \Sigma)\le \mathcal{N}(\varphi_\theta) \le \mathcal{N}(\varphi).
\]
\end{proof}

The encoding map $\varphi_\theta:\mathcal{X}\to \mathcal{F}$ is required to be homeomorphic, this adds strong topological constraints to the manifold $\Sigma$. For example, if $\Sigma$ is a surface, $\mathcal{F}$ is $\mathbb{R}^2$, then $\Sigma$ must be a genus zero surface with boundaries. In general, assume $\varphi_\theta(\Sigma)$ is a simply connected domain in $\mathcal{F}=\mathbb{R}^m$, then $\Sigma$ must be a $m$-dimensional topological disk. The topological constraint implies that autoencoder can only learn manifolds with simple topologies, or a local chart of the whole manifold.

On the other hand, the geometry and the embedding of $\Sigma$ determines the linear rectifiability of $(\Sigma,\mathbb{R}^n)$.

\begin{lemma}Suppose a $n$ dimensional manifold $\Sigma$ is embedded in $\mathbb{R}^{n+1}$,
\begin{diagram}
M&\rTo^{G}& \mathbb{S}^n &\rTo^{p} & \mathbb{RP}^n
\end{diagram}
where $G:\Sigma\to\mathbb{S}^n$ is the Gauss map, $\mathbb{RP}^n$ is the real projective space, the projection $p:\mathbb{S}^n\to\mathbb{RP}^n$ maps antipodal points to the same point, if $p\circ G(\Sigma)$ covers the whole $\mathbb{RP}^n$, then $\Sigma$ is not linear rectifiable.
\end{lemma}
\begin{proof} Given any unit vector $\mathbf{v}\in\mathbb{R}^{n+1}$, all the unit vectors orthogonal to $\mathbf{v}$ form a sphere $\mathbb{S}^{n-2}(\mathbf{v})$, then $p(\mathbb{S}^{n-2}(\mathbf{v}))\cap \mathbb{RP}^{n}\neq \emptyset$, therefore there is a point $q\in \Sigma$, $\mathbf{v}$ is in the tangent space at $q$. Line $q + t\mathbf{v}$ is tangent to $\Sigma$, by shifting the line by an infinitesimal amount, the line intersects $\Sigma$ at two points. This shows there is no linear mapping, which projects $\Sigma$ onto $\mathbb{R}^n$ along $\mathbf{v}$. Because $\mathbf{v}$ is arbitrary, $\Sigma$ is not linear rectifiable.
\end{proof}

\setlength{\tabcolsep}{6pt}
\begin{figure}[h!]
\begin{center}
\begin{tabular}{cccc}
\includegraphics[height=0.22\textwidth]{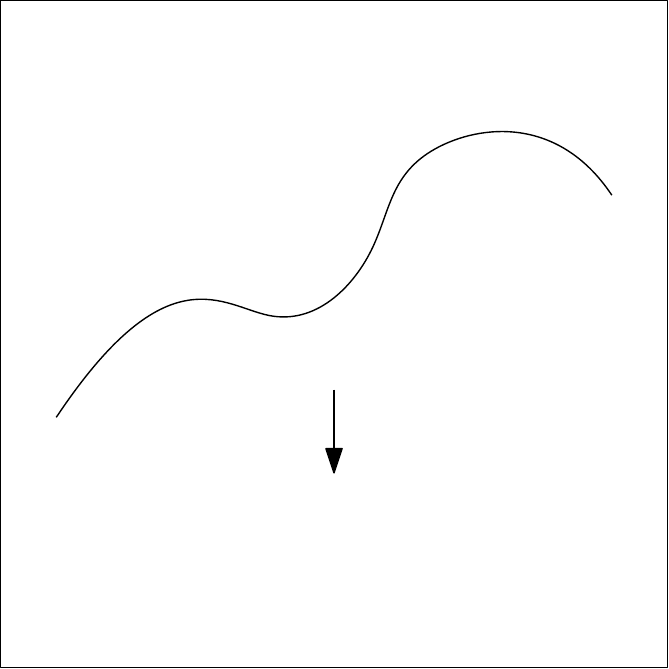}&
\includegraphics[height=0.22\textwidth]{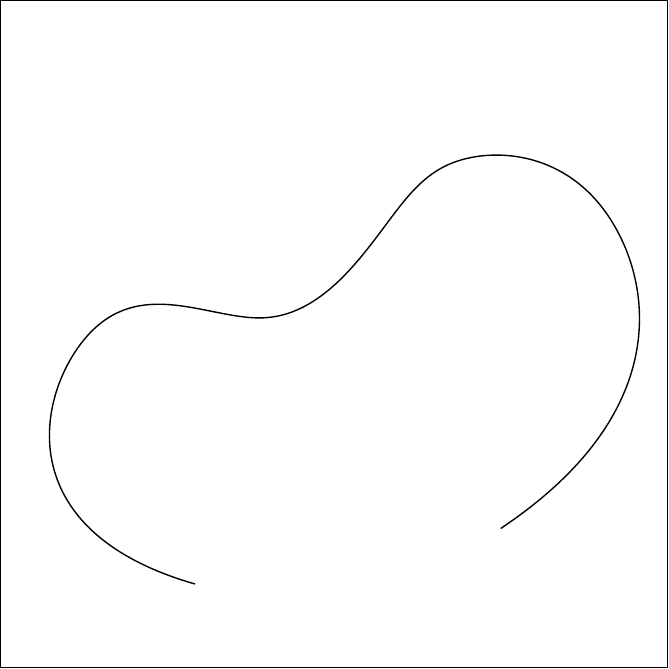}&
\includegraphics[height=0.22\textwidth]{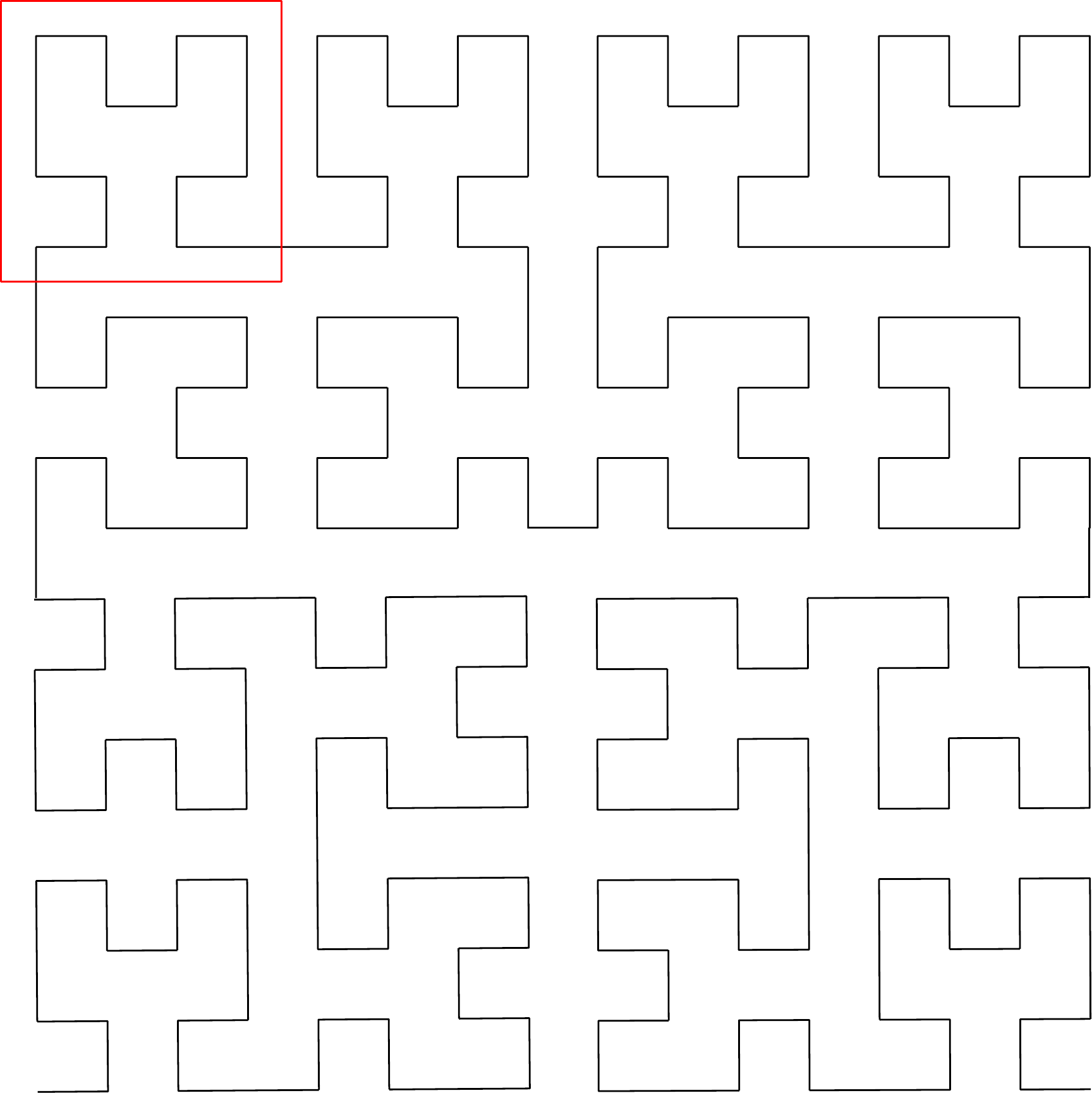}&
\includegraphics[height=0.22\textwidth]{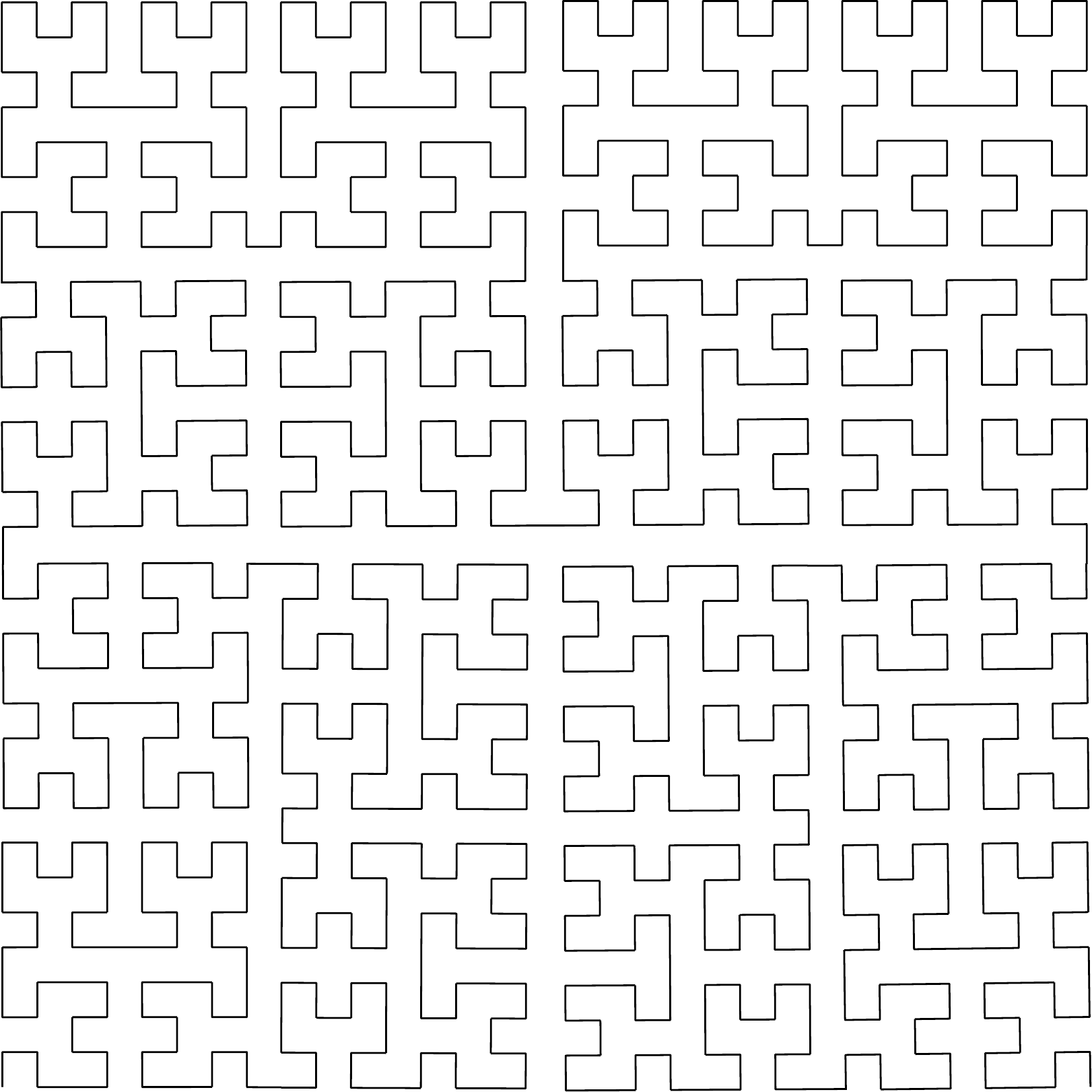}\\
a. linear rectifiable & b. non-linear-rectifiable  & c. $C_1$ Peano curve & d. $C_2$ Peano curve\\
\end{tabular}
\end{center}
\caption{Linear rectifiable and non-linear-rectifiable curves.}
\label{fig:peano_curves}
\end{figure}

\begin{theorem}Given any ReLU deep neural network $N(w_0,w_1,\dots,w_k,w_{k+1})$, there is a manifold $\Sigma$ embedded in $\mathbb{R}^{w_0}$, such that $\Sigma$ can not be encoded by $N$.
\label{thm:non_encodable}
\end{theorem}
\begin{proof} First, we prove the simplest case. When $(w_0,w_{k+1})=(2,1)$, we can construct space filling Peano curves,  as shown in Fig.~\ref{fig:peano_curves}. Suppose $C_1$ is shown in the left frame, we make $4$ copies of $C_1$, by translation, rotation, reconnection and scaling to construct $C_2$, as shown in the right frame. Similarly, we can construct all $C_k$'s. The red square shows one unit, $C_1$ has $16$ units, $C_n$ has $4^{n+1}$ units. Each unit is not rectifiable, therefore
\[
    \mathcal{N}(\mathbb{R}^2,C_n) \ge 4^{n+1}.
\]
We can choose $n$ big enough, such that $4^{n+1} > \mathcal{N}(N)$, then $C_n$ can not be encoded by $N$.

Similarly, for any $w_0$ and $w_{k+1}=1$, we can construct Peano curves to fill $\mathbb{R}^{w_0}$, which can not be encoded by $N$. The Peano curve construction can be generalized to higher dimensional manifolds by direct product with unit intervals.
\end{proof}

\section{Control Induced Measure}
\label{sec:measure_control}

\setlength{\tabcolsep}{0pt}
\begin{figure}[h!]
\begin{center}
\begin{tabular}{ccc}
\includegraphics[height=0.36\textwidth]{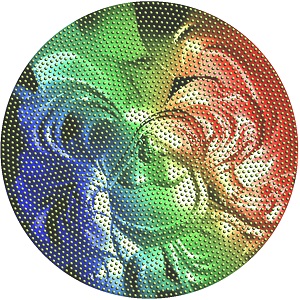}&
\includegraphics[height=0.36\textwidth]{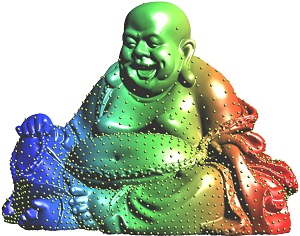}\\
\includegraphics[height=0.36\textwidth]{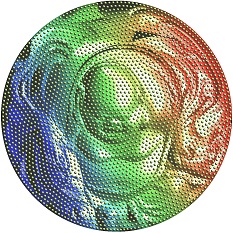}&
\includegraphics[height=0.36\textwidth]{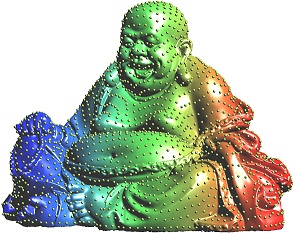}\\
\end{tabular}
\end{center}
\caption{Control distributions by optimal mass transportation.}
\label{fig:OMT}
\end{figure}

In generative models, such as VAE \cite{Kingma2013AutoEncodingVB} or GAN \cite{arjovsky2017wasserstein}, the probability measure in the latent space induced by the encoding mapping $(\varphi_\theta)_*\mu$ is controlled to be simple distributions, such as Gaussian or uniform, then in the generating process, we can sample from the simple distribution in the latent space, and use the decoding map to produce a sample in the ambient space.

The buddha surface $\Sigma$ is conformally mapped onto the planar unit disk $\varphi:\Sigma\to\mathbb{D}$ using the Ricci flow method \cite{DBLP:journals/pami/ZengSG10}, the uniform distribution on the parameter domain induces a non-uniform distribution on the surface, as shown in the top row of Fig.~\ref{fig:OMT}. Then by composing with an optimal mass transportation map $\psi:\mathbb{D}\to\mathbb{D}$ using the algorithm in \cite{DBLP:journals/pami/SuWSZS0G15}, one obtain an area-preserving mapping $\psi\circ\varphi:\Sigma\to\mathbb{D}$, the image is shown in the bottom row of Fig.~\ref{fig:OMT} left frame. Then we uniformly sample the planar disk to get the samples $\mathbf{Z}=\{\mathbf{z}_1,\dots,\mathbf{z}_k\}$, then pull them back on to $\Sigma$ by $\psi\circ\varphi$, $\mathbf{X}=\{\mathbf{x}_1,\dots,\mathbf{x}_k\}$, $x_i=(\psi\circ\varphi)^{-1}(z_i)$. Because $\psi\circ\varphi$ is area-preserving, $\mathbf{Z}$ is uniformly distributed on the disk, $\mathbf{X}$ is uniformly distributed on $\Sigma$ as shown in the bottom row of Fig.~\ref{fig:OMT} right frame.

\paragraph*{Optimal Mass Transportation} The optimal transportation theory can be found in Villani's classical books \cite{vil}\cite{villani2008optimal}. Suppose $\nu=(\varphi_\theta)_*\mu$ is the induced probability in the latent space with a convex support $\Omega\subset \mathcal{F}$, $\zeta$ is the simple distribution, e.g. the uniform distribution on $\Omega$. A mapping $T:\Omega\to\Omega$ is measure-preserving if $T_*\nu = \zeta$. Given the transportation cost between two points $c:\Omega\times\Omega\to\mathbb{R}$, the transportation cost of $T$ is defined as
\[
    \mathcal{E}(T):= \int_{\Omega} c(\mathbf{x},T(\mathbf{x})) d\nu(\mathbf{x}).
\]
 The \emph{Wasserstein distance} between $\nu$ and $\zeta$ is defined as
\[
    \mathcal{W}(\nu,\zeta):= \inf_{T_*\nu=\zeta} \mathcal{E}(T).
\]
The measure-preserving map $T$ that minimizes the transportation cost is called the \emph{optimal mass transportation map}.

Kantorovich proved that the Wasserstein distance can be represented as
\[
    \mathcal{W}(\nu,\zeta):= \max_{f} \int_\Omega f d\nu + \int_\Omega f^c d\zeta
\]
where $f:\Omega\to \mathbb{R}$ is called the Kontarovhich potential, its c-transform
\[
    f^c(\mathbf{y}):= \inf_{\mathbf{x}\in\Omega}{c(\mathbf{x},\mathbf{y})-f(\mathbf{x})}.
\]

In WGAN, the discriminator computes the generator computes the decoding map $\psi_\theta:\mathcal{F}\to\mathcal{X}$, the discriminator computes the Wasserstein distance between $(\psi_\theta)_*\zeta$ and $\mu$. If the cost function is chosen to be the $L^1$ norm, $c(\mathbf{x},\mathbf{y})=|\mathbf{x}-\mathbf{y}|$, $f$ is 1-Lipsitz, then $f^c=-f$, the discriminator computes the Kontarovich potential, the generator computes the optimal mass transportation map, hence WGAN can be modeled as an optimization
\[
    \min_\theta \max_f \int_\Omega f \circ \psi_\theta(\mathbf{z}) d\zeta(\mathbf{z}) - \int_\mathcal{X} f(x) d\mu(x).
\]
The competition between the discriminator and  the generator leads to the solution.

If we choose the cost function to be the $L^2$ norm, $c(\mathbf{x},\mathbf{y})=\frac{1}{2}|\mathbf{x}-\mathbf{y}|^2$, then the computation can be greatly simplified. Briener's theorem \cite{brenier} claims that there exists a convex function $u:\Omega\to \mathbb{R}$, the so-called Brenier's potential, such that its gradient map $\nabla u:\Omega\to\Omega$ gives the optimal mass transportation map. The Brenier's potential satisfies the Monge-Ampere equation
\[
    \text{det}\left(\frac{\partial^2 u }{\partial x_i \partial x_j}\right) = \frac{\nu(\mathbf{x})}{\zeta(\nabla u(\mathbf{x}))}.
\]
Geometrically, the Monge-Ampere equation can be understood as solving Alexandroff problem: finding a convex surface with prescribed Gaussian curvature. A practical algorithm based on variational principle can be found in \cite{Gu_AJM_2016}.
The Brenier's potential and the Kontarovich's potential are related by the closed form
\begin{equation}
    u(\mathbf{x})=\frac{1}{2}|\mathbf{x}|^2 - f(\mathbf{x}).
    \label{eqn:relation}
\end{equation}
Eqn.(\ref{eqn:relation}) shows that: the generator computes the optimal transportation map $\nabla u$, the discriminator computes the Wasserstein distance by finding Kontarovich's potential $f$; $u$ and $f$ can be converted to each other, hence the competition between the generator and the discriminator is unnecessary, the two deep neural networks for the generator and the discriminator are redundant.

\paragraph*{Autoencoder-OMT model}

\setlength{\tabcolsep}{1mm}
\begin{figure}[h!]
\begin{center}
\begin{tabular}{c}
\epsfig{file=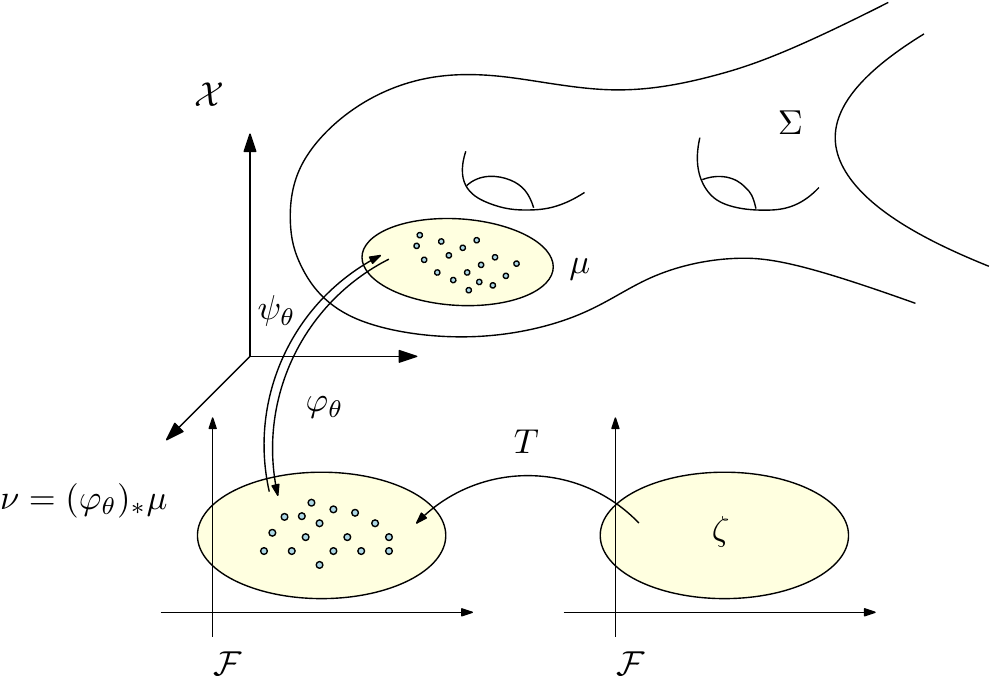, height=0.4\textwidth}
\end{tabular}
\end{center}
\caption{Autoencoder combined with a optimal transportation map.\label{fig:AE_OMT}}
\end{figure}
As shown in Fig.~\ref{fig:AE_OMT}, we can use autoencoder to realize encoder $\varphi_\theta:\mathcal{X}\to\mathcal{F}$ and decoder $\psi_\theta:\mathcal{F}\to\mathcal{X}$, use OMT in the latent space to realize probability transformation $T:\mathcal{F}\to\mathcal{F}$, such that
\[
    T_*\zeta = (\varphi_\theta)_*\mu.
\]
We call this model as OMT-autoencoder.

\setlength{\tabcolsep}{1mm}
\begin{figure}[h!]
\begin{center}
\begin{tabular}{cc}
\epsfig{file=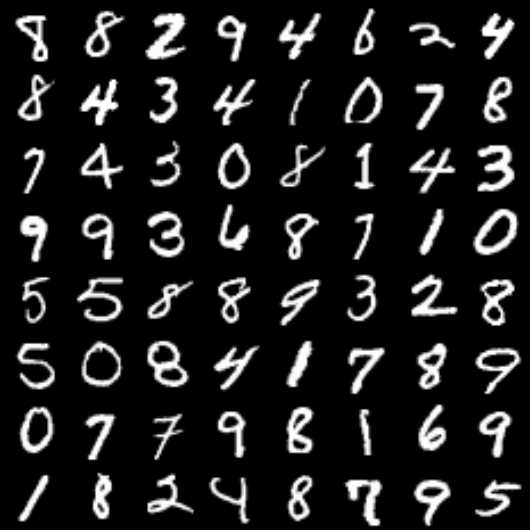, height=0.325\textwidth}&
\epsfig{file=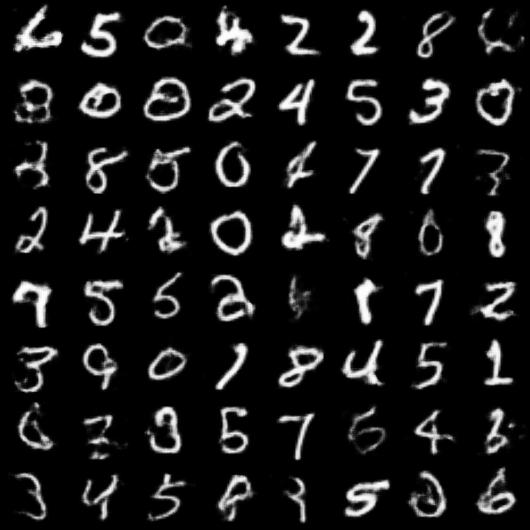, height=0.325\textwidth}\\
(a) real digits & (b) VAE\\
\epsfig{file=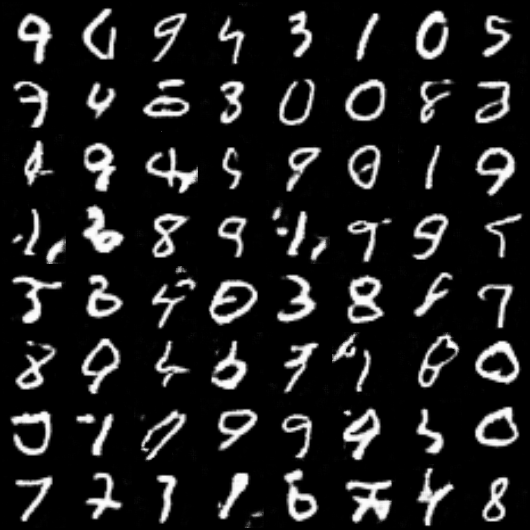, height=0.325\textwidth}&
\epsfig{file=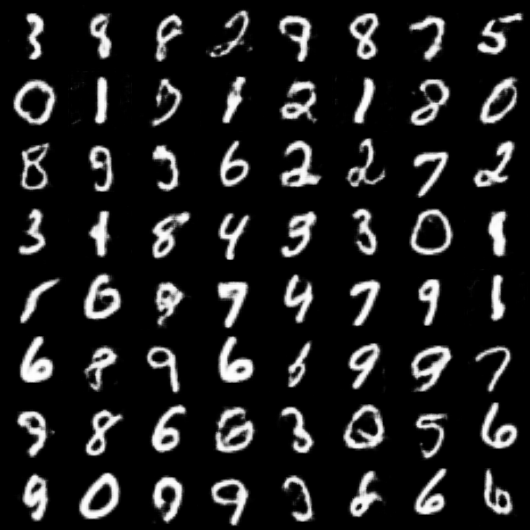, height=0.325\textwidth}\\
(c) WGAN & (d) AE-OMT
\end{tabular}
\end{center}
\label{fig:mnist}
\end{figure}

Fig.~\ref{fig:mnist} shows the experiments on the MNIST data set. The digits generates by OMT-AE have better qualities than those generated by VAE and WGAN. Fig.(\ref{fig: CelebA}) shows the human facial images on CelebA data set. The images generated by OMT-AE look better than those produced by VAE.

\setlength{\tabcolsep}{1mm}
\begin{figure}[h!]
\begin{center}
\begin{tabular}{cc}
\epsfig{file=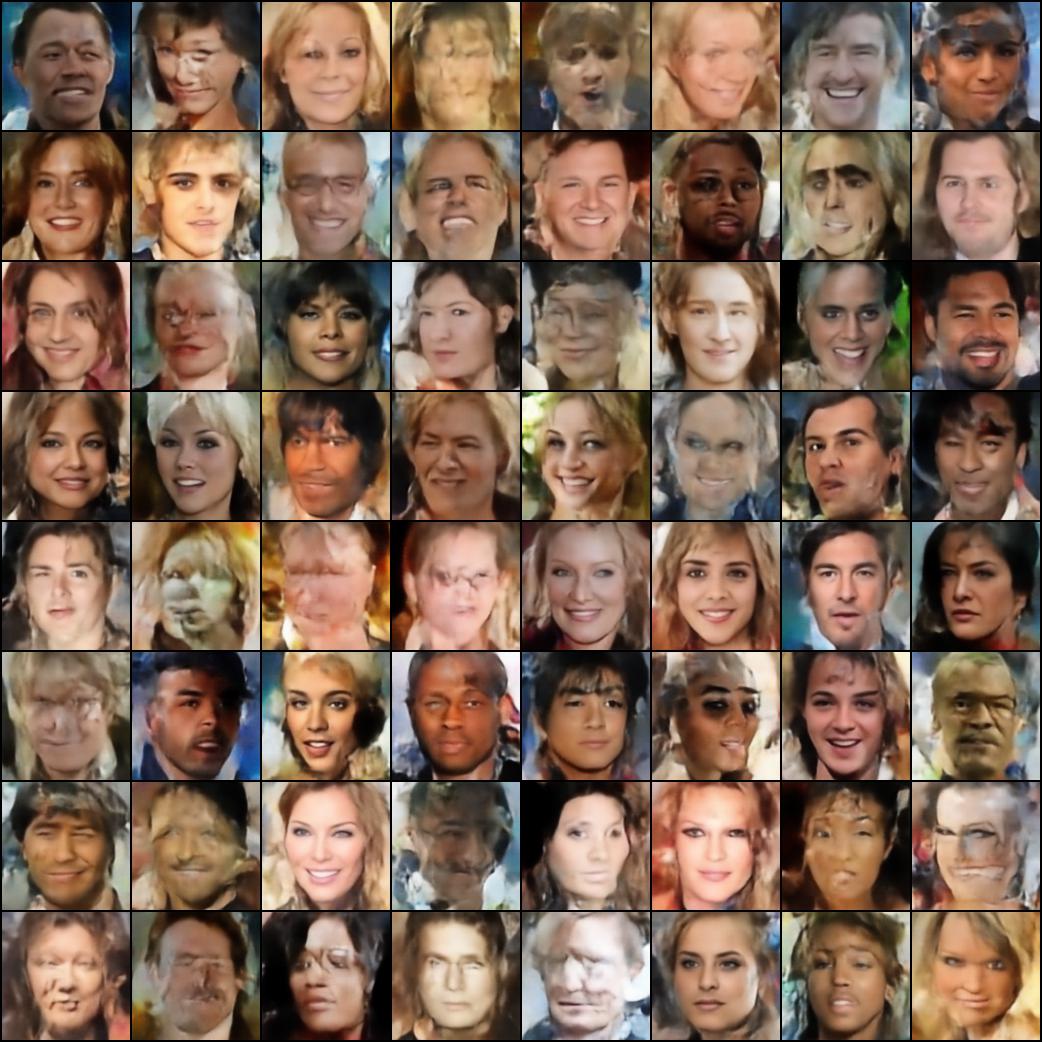, height=0.45\textwidth}&
\epsfig{file=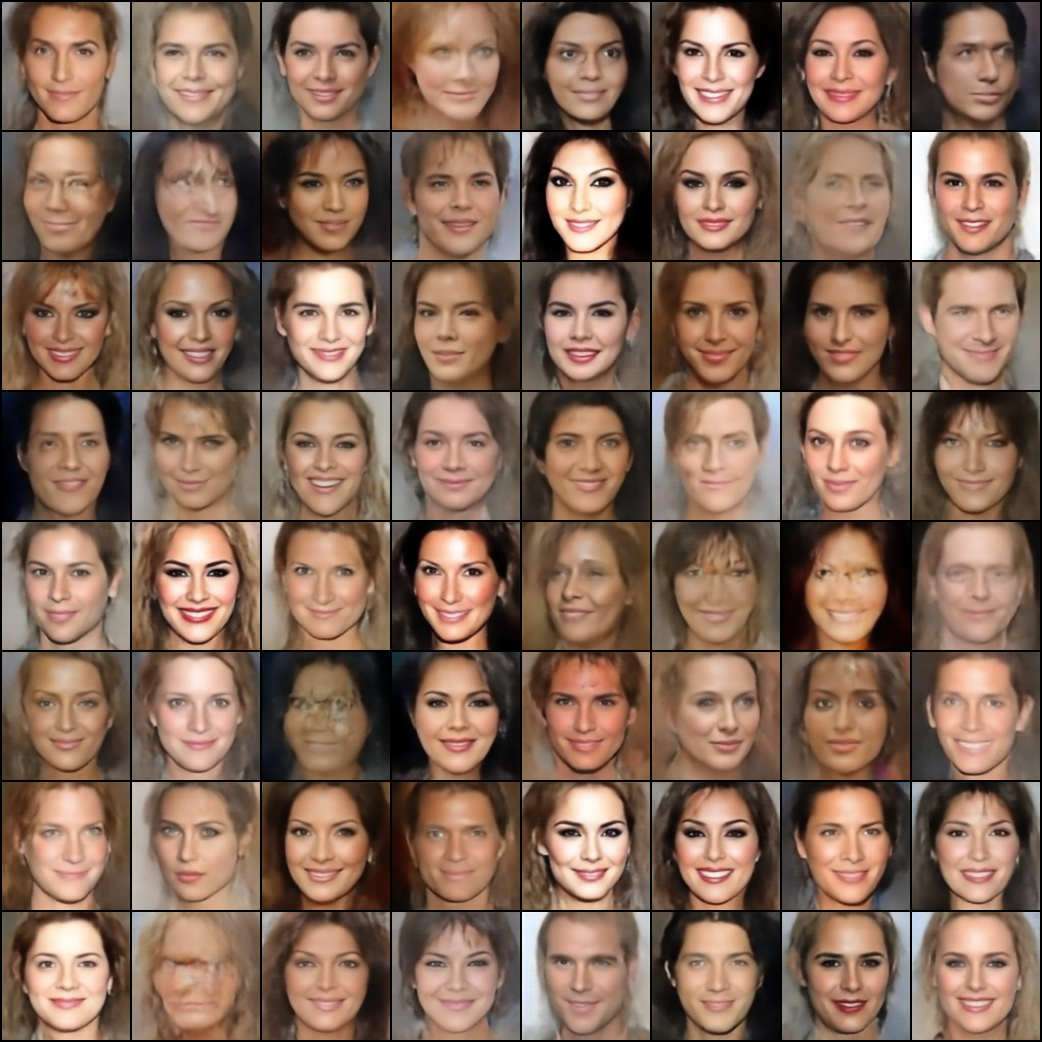, height=0.45\textwidth}\\
(a) VAE & (d) AE-OMT
\end{tabular}
\end{center}
\label{fig: CelebA}
\end{figure}

\section{Conclusion}

This work gives a geometric understanding of autoencoders and general deep neural networks. The underlying principle is the manifold structure hidden in data, which attributes to the great success of deep learning. The autoencoders learn the manifold structure and construct a parametric representation. The concepts of rectified linear complexities are introduced to both DNN and manifold, which describes the fundamental learning limitation of the DNN and the difficulty to be learned of the manifold. By applying the concept of complexities, it is shown that for any DNN with fixed architecture, there is a manifold too complicated to be encoded by the DNN. Experiments on surfaces show the approximation accuracy can be improved. By applying $L^2$ optimal mass transportation theory, the probability distribution in the latent space can be fully controlled in a more understandable and more efficient way.

In the future, we will develop refiner estimates for the complexities of the deep neural networks and the embedding manifolds, generalize the geometric framework to other deep learning models.

\section*{Appendix}
\label{sec:appendix}
Here we illustrate some examples and explain the implementation details.

\setlength{\tabcolsep}{4pt}
\begin{figure}[h!]
\begin{center}
\begin{tabular}{cccc}
\includegraphics[width=0.22\textwidth]{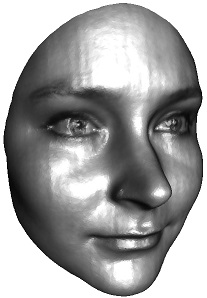}&
\includegraphics[width=0.30\textwidth]{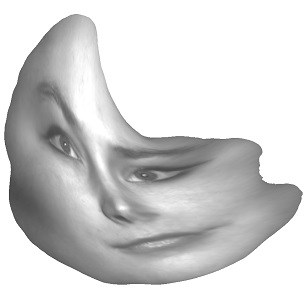}&
\includegraphics[width=0.22\textwidth]{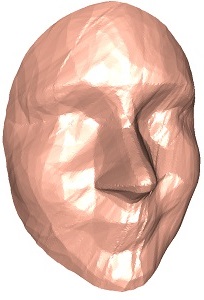}&
\includegraphics[width=0.22\textwidth]{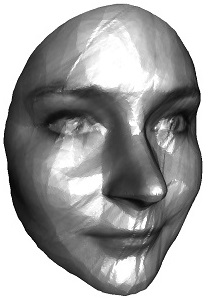}\\
input manifold& latent representation& reconstructed manifold& reconstructed manifold
\end{tabular}
\end{center}
\caption{A human facial surface is encoded/decoded by an autoencoder.}
\label{fig:sophie}
\end{figure}

\paragraph*{Facial Surface} Fig.~\ref{fig:sophie} shows a human facial surface $\Sigma$ is encoded/decoded by an autoencoder. From the image, it can be seen that the encoding/decoding maps are homeomorphic. The Hausdorff distance between the input surface and the reconstructed surface is relatively small, but the normal deviation is big. The geometric details around the mouth area are lost during the process. There are a lot of local curvature fluctuations. Furthermore, the shape of the encoding image (parameter image) in the latent space is highly irregular, this creates difficulty for generating random samples on the reconstructed manifold.

\setlength{\tabcolsep}{0pt}
\begin{figure}[h!]
\begin{center}
\begin{tabular}{ccc}
\includegraphics[height=0.31\textwidth]{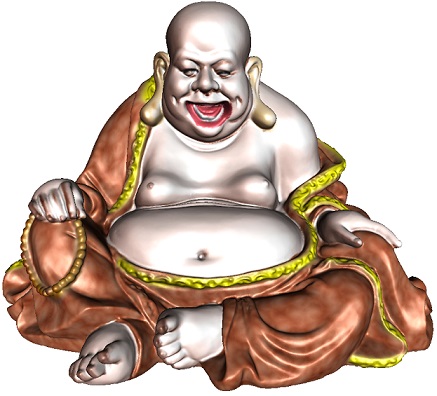}&
\includegraphics[height=0.31\textwidth]{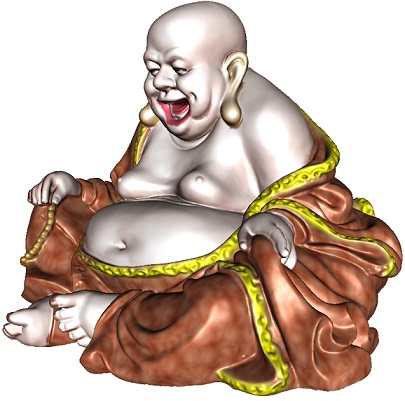}&
\includegraphics[height=0.31\textwidth]{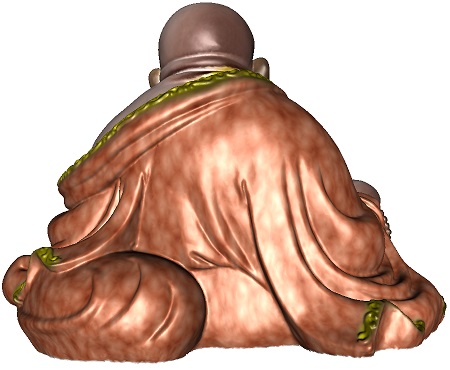}\\
a. front  view& b. left  view& c. back view\\
\includegraphics[height=0.31\textwidth]{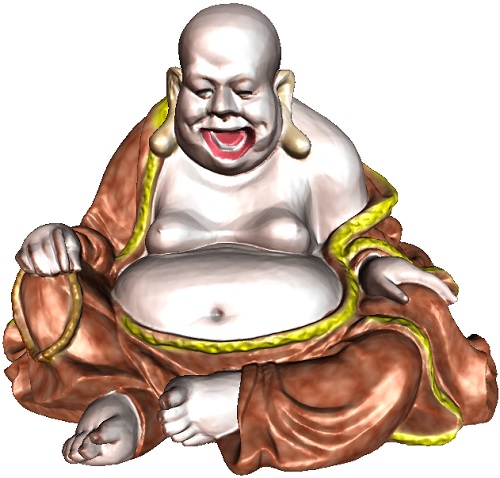}&
\includegraphics[height=0.31\textwidth]{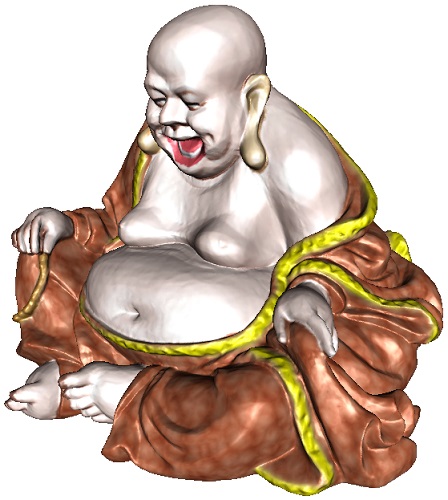}&
\includegraphics[height=0.31\textwidth]{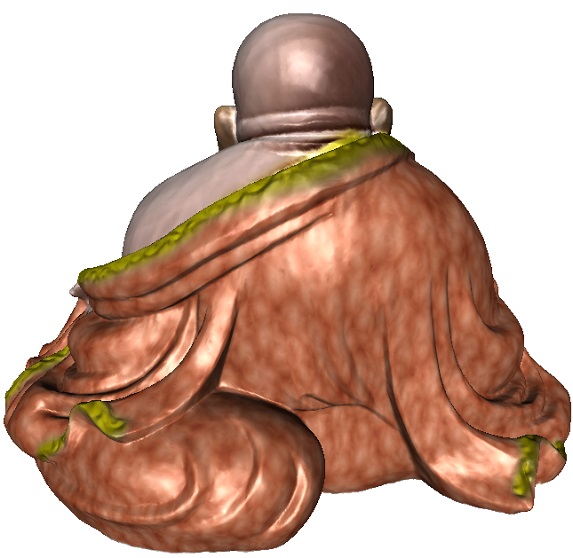}\\
d. right view&e. front view&f. back view\\
\end{tabular}
\end{center}
\caption{top row: input manifold;bottom row, reconstructed manifold.}
\label{fig:buddha}
\end{figure}

\paragraph*{Buddha Model}
Fig.~\ref{fig:buddha} shows the buddha model, the top row shows the three views of the input surface, the bottom row shows the reconstructed surface. The encoder network architecture is $\{3,768,384,192,96,48,2\}$, the decoder network is $\{2,48,96,192,384, 768,3\}$. The input and output spaces are $\mathbb{R}^3$, the latent space is $\mathbb{R}^2$. We use ReLU as the activation function in hidden layers except the latent space layer. The loss function is the mean squared error between the input and the target. Adam optimizer is used in this autoencoder and the weight decay is set to $0$ in the optimizer. From the figure, we can see the reconstruction approximates the original surface with high accuracy, all the subtle geometric features are well preserved. We uniformly sample the surface, there are $235,771$ samples in total. The number of cells in the cell decomposition induced by the reconstruction map is $230051$. We see that the autoencoder produces a highly refined cell decomposition to capture all the geometric details of the input surface. The source code and the data set can be found in \cite{Code}. If we reduce the number of neurons and add regularize the output surface, then the reconstructed surface loses geometric details, and preserves the major shape as shown in Fig.~\ref{fig:recondstructed_manifold_cell_decomposition}. Furthermore, the mapping is not homeomorphic either, near the mouth and finger areas, the mapping is degenerated.

%
%
%
\setlength{\tabcolsep}{0pt}
\begin{figure}[h!]
\begin{center}
\begin{tabular}{ccc}
\includegraphics[height=0.3\textwidth]{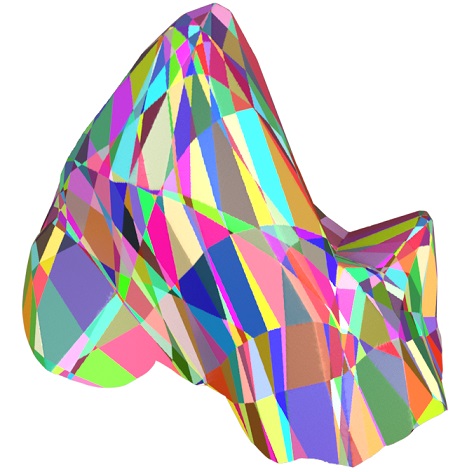}&
\includegraphics[height=0.3\textwidth]{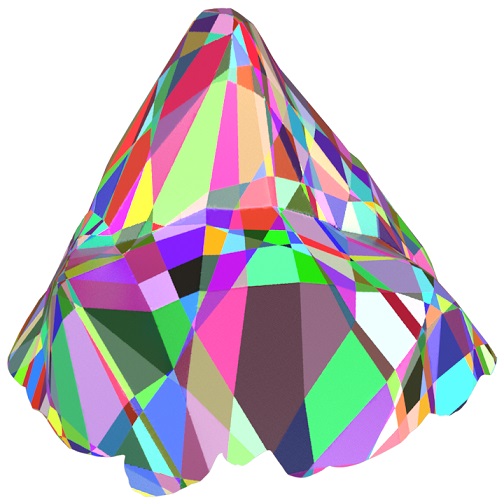}&
\includegraphics[height=0.3\textwidth]{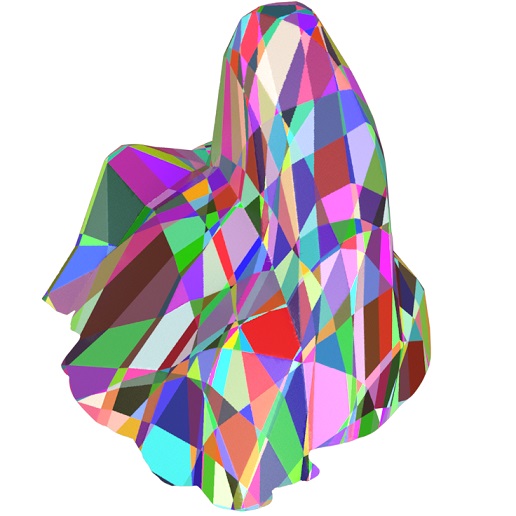}\\
a. right view & b. front view & back view\\
\end{tabular}
\end{center}
\caption{Reconstructed manifold with cell decomposition produced by an autoencoder with half of the neurons.}
\label{fig:recondstructed_manifold_cell_decomposition}
\end{figure}

\section*{Acknowledgement}

The authors thank our students: Yang Guo, Dongsheng An, Jingyao Ke, Huidong Liu for all the experimental results, also thank our collaborators:  Feng Luo, Kefeng Liu, Dimitris Samaras for the helpful discussions.

\newpage

\end{document}